%% file: paper.tex
\newif\ifarxiv
\arxivtrue    %

\ifarxiv
\documentclass[a4paper]{article}
\else
\documentclass[journal]{IEEEtran}
\fi

\input{preamble}
\title{Escaping Local Optima using Crossover with Emergent or Reinforced Diversity}

\ifarxiv
\author[1]{Duc-Cuong~Dang}
\author[2]{Tobias~Friedrich}
\author[2]{Timo~K{\"o}tzing}
\author[2]{Martin~S.~Krejca}
\author[1]{Per~Kristian~Lehre}
\author[3]{Pietro~S.~Oliveto}
\author[3]{Dirk~Sudholt}
\author[2]{Andrew~M.~Sutton}
\affil[1]{University of Nottingham, United Kingdom}
\affil[2]{Hasso Plattner Institute, Potsdam, Germany}
\affil[3]{University of Sheffield, United Kingdom}
\else
\author{Duc-Cuong~Dang,
        Tobias~Friedrich,
        Timo~K{\"o}tzing,
        Martin~S.~Krejca,
        Per~Kristian~Lehre,
        Pietro~S.~Oliveto,
        Dirk~Sudholt,
        Andrew~M.~Sutton
\thanks{D.-C. Dang and P. K. Lehre are with the ASAP Research Group
        at the School of Computer Science, University of Nottingham, United Kingdom.}
\thanks{T.~Friedrich, T.~K{\"o}tzing, M.~S.~Krejca and A.~M.~Sutton are with the Chair of Algorithm Engineering,
        Hasso Plattner Institute, Potsdam, Germany.}
\thanks{P.~S.~Oliveto and D.~Sudholt are with the Department of Computer Science,
        University of Sheffield, United Kingdom.}
}
\fi

\begin{document}
\maketitle

\begin{abstract}
\ifarxiv
Population diversity is essential for avoiding premature convergence in Genetic Algorithms (GAs) and for the effective use of crossover. Yet the dynamics of how diversity emerges in populations are not well understood.
We use rigorous run time analysis to gain insight into population dynamics and GA performance for the \muGA and the $\jump_k$ test function.
We show that the interplay of crossover followed by mutation may serve as a catalyst leading to a sudden burst of diversity. This leads to improvements of the expected optimisation time of order $\Omega(n/\log n)$ compared to mutation-only algorithms like the \oneoneea.
Moreover, increasing the mutation rate by an arbitrarily small constant factor can facilitate the generation of diversity, leading to speedups of order~$\Omega(n)$.
We also compare seven commonly used diversity mechanisms and evaluate their impact on run time bounds for the \muGA.
All previous results in this context
only hold for unrealistically low crossover probability $\pc=\bigOh{k/n}$,
while we give analyses for the setting of constant $\pc < 1$ in all but one case.
For the typical case of constant $k > 2$ and constant $\pc$,
we can compare the resulting expected optimisation times for different diversity mechanisms assuming
an optimal choice of~$\mu$:
\begin{itemize}
\item
$\bigOh{n^{k-1}}$ for
duplicate elimination/minimisation,
\item
$\bigOh{n^2\log n}$ for
maximising the convex hull,
\item
$\bigOh{n\log n}$ for
deterministic crowding (assuming $\pc = k/n$),
\item
$\bigOh{n\log n}$ for
maximising the Hamming distance,
\item
$\bigOh{n\log n}$ for
fitness sharing,
\item
$\bigOh{n\log n}$ for
the single-receiver island model.
\end{itemize}
This proves a sizeable advantage of all variants of the \muGA compared to the
(1+1)~EA, which requires time $\Theta(n^k)$.
Experiments complement our theoretical findings and further highlight the
benefits of crossover and diversity on $\jump_k$.
\else
Population diversity is essential for avoiding premature convergence in Genetic Algorithms and for the effective use of crossover. Yet the dynamics of how diversity emerges in populations are not well understood.
We use rigorous run time analysis to gain insight into population dynamics and Genetic Algorithm performance for the ($\mu$+1)~Genetic Algorithm and the $\jump$ test function.
We show that the interplay of crossover followed by mutation may serve as a catalyst leading to a sudden burst of diversity. This leads to significant improvements of the expected optimisation time compared to mutation-only algorithms like the (1+1)~Evolutionary Algorithm.
Moreover, increasing the mutation rate by an arbitrarily small constant factor can facilitate the generation of diversity, leading to even larger speedups.

We also compare seven commonly used diversity mechanisms and evaluate their impact on run time bounds for the ($\mu$+1)~Genetic Algorithm:
duplicate elimination/minimisation,
maximising the convex hull,
deterministic crowding,
maximising the Hamming distance,
fitness sharing,
and the single-receiver island model.
All mechanisms provide a sizeable advantage of all variants of the ($\mu$+1)~Genetic Algorithm compared to the
(1+1)~Evolutionary Algorithm.
Experiments complement our theoretical findings and further highlight the
benefits of crossover and diversity on $\jump$.
\fi
\end{abstract}

\ifarxiv
\else
\fi

\section{Introduction}\label{sec:intro}

Genetic Algorithms (GAs) are powerful general-purpose optimisers that perform surprisingly well in many applications, including those where the problem is not well understood to apply a tailored algorithm. Their wide-spread success is based on a number of factors: using populations to diversify search, using mutation to generate novel solutions, and using crossover to combine features of good solutions.

Pr{\"u}gel-Bennett~\cite{PruegelBennett2010} gives several reasons for the success of populations and crossover. Crossover can combine building blocks of good solutions and help to focus search on bits where parents disagree~\cite{PruegelBennett2010}. For both tasks, the population needs to be diverse enough; without sufficient diversity in the population, crossover is not effective. A common problem in the application of GAs is the loss of diversity when the population converges to copies of the same search point, often called \emph{premature convergence}.
Understanding how populations gain and lose diversity during the course of the optimisation is vital for understanding the working principles of GAs and for tuning the design of GAs to get the best possible performance.

Rigorous run time analysis has emerged as a powerful theory that has provided many insights into the performance of GAs~\cite{AugerD11,BookNeuWit,Jansen2013,OlivetoWitt2014,OlivetoWitt2015,Corus2014}, including the benefit of crossover~\cite{Jansen2002Crossover,Lehre2011Crossover,OlivetoHeYao2008,Doerr2012,KotzingST:c:11:crossover,Sudholt2012b}.
It has guided algorithm design, including the discovery of new variants of GAs such as the (1+($\lambda$,$\lambda$))~GA~\cite{Doerr2015}, which has shown very good performance across a range of hard problems~\cite{goldman:2015:fastp3}.

However, understanding population diversity and crossover has proved elusive. The first example function where crossover was proven to be beneficial is called $\jump_k$. In this problem, GAs have to overcome a fitness valley such that all local optima have Hamming distance~$k$ to the global optimum. Jansen and Wegener~\cite{Jansen2002Crossover} showed that, while mutation-only algorithms such as the \oneoneea require expected time $\Theta(n^k)$,  a simple \muGA with crossover only needs time $\bigOh{\mu n^2 k^3 + 4^k/\pc}$. This time is $\bigOh{4^k/\pc}$ for large~$k$, and hence significantly faster than mutation-only GAs. However, their analysis requires an unrealistically small crossover probability $\pc \le 1/(ckn)$ for a large constant~$c > 0$.

K{\"o}tzing, Sudholt, and Theile~\cite{KotzingST:c:11:crossover} later refined these results towards a crossover probability $\pc \le k/n$, which is still unrealistically small.
Both approaches focus on creating diversity through a sequence of lucky mutations, relying on crossover to create the optimum, once sufficient diversity has been created. Their arguments break down if crossover is applied frequently. Hence, these analyses do not reflect the typical behaviour in GA populations with constant crossover probabilities $\pc = \Theta(1)$ as used in practice.

Lehre and Yao analysed the run time of the \muGA with deterministic
crowding for arbitrary crossover rates $\pc>0$, showing exponential
run time gaps between the case $\pc=0$ and $\pc>0$
\cite{Lehre2011Crossover}. The gain in performance in that analysis
stems from the ability of a diverse population to optimise multiple,
separated paths in parallel using a diversity-preservation
mechanism. Similar results have been also shown for instances of the vertex cover problem by generating diversity either through deterministic crowding \cite{OlivetoHeYao2008} or through island models \cite{Neumann2011}.
Here, we will consider a different effect.

We provide a novel approach loosely inspired from population genetics: we show that diversity can also be created by crossover, followed by mutation. Note that the perspective of crossover creating diversity is common in population genetics~\cite{Weissman1389,Komarova}. A frequent assumption is that crossover mixes all alleles in a population, leading to a situation called \emph{linkage equilibrium}, where the state of a population is described by the frequency of alleles~\cite{Barton:2013:QPG:2463372.2463568}.

For the maximum crossover probability $\pc=1$, we show that on
$\jump_k$ diversity emerges naturally in a population: the interplay
of crossover, followed by mutation, can serve as a catalyst for
creating a diverse range of search points out of few different
individuals. This naturally emerging diversity allows to prove a
speedup of order $n/\log n$ for $k \ge 3$ and standard mutation rate $\varpm=1/n$
compared to mutation-only
algorithms such as the \oneoneea. Increasing the mutation rate to
$\varpm=(1+\delta)/n$ for an arbitrarily small constant $\delta>0$,
leads to a speedup of order $n$.

Both operators are proven to be vital: mutation requires $\Theta(n^k)$ expected iterations to hit the optimum from a local optimum.
Also using crossover on its own does not help much.
As shown in~\cite[Theorem~8]{KotzingST:c:11:crossover}, using only crossover with $\pc = \Omega(1)$ but no mutation following crossover, diversity reduces quickly, leading to inefficient running times for small population sizes ($\mu = \bigOh{\log n}$).

In this paper, we also rigorously study varieties of diversity mechanisms of the \muGA on the $\jumpK$ function. We consider three varieties of diversity-preserving mechanisms: (1) methods that encourage \emph{genotypic uniqueness}, (2) methods that increase \emph{genotypic distance}, and (3) methods that employ parallelism. Our results are summarised in Table~\ref{tab:results}. The definition of the mechanisms with their variety are detailed in Section~\ref{sec:prelim}. The table presents the run time bound for a best-possible population size~$\mu$ for each diversity mechanism.

We observe on the $\jumpK$ function how a diverse population spreads on local optima and eventually manages to escape them. Our results show that already small changes in the  tie-breaking rule of the \muGA can make a big difference. We rigorously prove asymptotic run time upper bounds for different diversity mechanisms. The respective proofs themselves give insight into \emph{how} the diversity mechanisms help escaping local optima.

All our analyses are based on observing the dynamic behaviour of the size of the largest \emph{species}, referring to a collection of identical genotypes as species. A population contains no diversity when only one species is present. However, mutation can create further species, and then the combination of crossover and mutation is able to rapidly create further species in a highly stochastic process. This diversity can then be exploited to find the global optimum on $\jump_k$ efficiently. A higher mutation rate facilitates the generation of new species and leads to better performance, with respect to rigorous upper run time bounds and empirical performance.

Using $\jump_k$ as a case study, our analyses shed light on how diversity emerges in populations and how to facilitate the emergence of diversity by tuning the mutation rate or by making use of specific diversity mechanisms.

\input{results-table}

Parts of the results were published in~\cite{divpaper} and~\cite{emergence}. Here we extend the analysis of no diversity mechanisms in~\cite{emergence} to higher mutation rates, leading to the surprising conclusion that increasing the mutation rates leads to smaller runtime bounds, compared to the standard mutation rate~$1/n$. Furthermore, the analysis of standard mutation rates in~\cite{emergence} was restricted to very short jumps, $k=O(1)$. Here we generalise the results to a much larger class of $\jump_k$ functions, only requiring $k=o(n)$.

\section{Preliminaries}\label{sec:prelim}
\input{preliminaries}

For the sake of completeness, in the next section, we provide the time bounds for the population to reach the plateau for the different algorithmic settings.

\section{Time to plateau}\label{sec:time-to-plateau}

\input{time-to-plateau}
\section{Natural diversity}\label{sec:no-mechanism}
We are first interested in the natural setting of Algorithm~\ref{alg:mu+1-GA}, \ie with $\pc=1.0$ and without any tailored tie-breaking rule, that is, breaking ties uniformly at random.

\subsection{Population dynamics}\label{sec:pop-dyn}

\input{population-dynamic}

\subsection{Standard mutation rate}\label{sec:std-mut}
\input{standard-mutation}

\subsection{High mutation rates}\label{sec:high-mut}
\input{high-mutation}
\section{Reinforced diversity}\label{sec:mecanism}
The aim of this section is to analyse how different tie-breaking rules in line \ref{li:tie-breaking} of Algorithm~\ref{alg:mu+1-GA} affect further the population diversity and the expected time to optimise the function.

We first consider diversity mechanisms that correspond to tie-breaking rules that try to promote having unique individuals in the population either by explicitly eliminating duplicates or by enforcing selection to occur between parent and offspring, who are more likely to be non-unique.
\subsection{Duplicate elimination}\label{sec:dupl-elim}
\input{duplicate-elimination}

\subsection{Duplicate minimisation}\label{sec:dupl-min}
\input{duplicate-minimization}

\subsection{Deterministic crowding}\label{sec:determ-crowd}
\input{deterministic-crowding}

If we have $\mu \leq n/k$, then it is possible for the population to be perfectly spread, i.e., any two individuals do not share a position with a $0$. We will see that this state is obtained quickly with the next three diversity mechanisms. Once any two individuals do not share a position with a $0$, every crossover operation has a probability of $4^{-k}$ to generate the optimum.
\subsection{Convex-hull maximisation}\label{sec:convex-hull}
\input{convex-hull}

\subsection{Total Hamming distance}\label{sec:total-distance}
\input{total-hamming-distance}

\subsection{Fitness sharing}\label{sec:fitness-sharing}
\input{fitness-sharing}

The last diversity mechanism that we look at relies on isolation/parallelism to promote population diversity.
\subsection{Island model}\label{sec:island-model}
\input{island-model}
\section{Experiments}\label{sec:experiment}
Since the theoretical results presented in the previous section are asymptotic
and they only provide upper bounds on the run time of the algorithms, we also
implemented the \muGA with the different diversity mechanisms and conducted experiments on
$\jump_k$ for various values of $k$, $n$, and $\varpm$.

In each tested setting of the algorithm and the function, the run is replicated
$100$ times with different random seeds. The number of function evaluations,
denoted `\# evaluations', is reported as the run time. The population size is
set to $\mu = c \ln{n}$, in which $c$ is chosen so that $\mu$ is not too small
for some small $n$, i.e., $c = 4\eulerE$.

\subsection{Impact of crossover and mutation rates}\label{sec:experiment-impact-crossover}
\input{experiment-crossover}

\subsection{Impact of diversity mechanisms}\label{sec:experiment-impact-diversity}
\input{experiment-diversity}

\section{Conclusion}\label{sec:concl}
A rigorous analysis of the \muGA has been presented showing how the use of, first, only crossover and mutation and, second, the use of diversity mechanisms considerably speeds up the run time for $\jump_k$ compared to algorithms using mutation only.

With regard to the first setting, traditionally it has been believed  that crossover may be useful only if sufficient diversity is readily available and that the emergence of diversity in the population is due to either mutation alone or should be enforced by the introduction of diversity mechanisms \cite{Friedrich:2009:ADM:1668000.1668003, OlivetoZarges2015}. Indeed, previous work highlighting that crossover may be beneficial for $\jump_k$ used unrealistically low crossover probabilities to allow mutation alone to create sufficient diversity. Conversely, our analysis shows that the interplay between crossover and mutation on the plateau of local optima of the $\jump_k$ function quickly leads to a burst of diversity that is then exploited by both operators to reach the global optimum. The right balance between the amount of mutation and crossover impacts the run time considerably.

While mutation rates lower than the standard $1/n$ rate considerably increase the expected run time, rates that are slightly higher than $1/n$ lead to improved performance. These rates also depend on the presence of crossover. For instance, for $k=4$, the best rate for a mutation-only algorithm is $4/n$ while the best rate for the \muGA with $\pc=1$ is considerably lower than $4/n$ and higher than $1/n$. Our experiments also reveal that the expected run time of the \muGA does not increase considerably with the increase of the gap length $k$. Hence, it is an open problem to provide tighter bounds on the expected run time than those presented in this paper.

Second, we have considered the role of selection-based diversity mechanisms used together with crossover for escaping local optima. We prove rigorous upper bounds on the run time of the \muGA for seven well-known diversity mechanisms optimising the $\jumpK$ function. Our results reveal a qualitative difference in the ability of the different diversity mechanisms to escape local optima.

In contrast to previous theoretical work on crossover for $\jumpK$, our upper bounds do not rely on unreasonably small (e.g., vanishing with $n$) crossover probabilities but instead cover the more practical case of constant crossover probabilities. Furthermore, our proofs provide insight into the ways that diversity mechanisms, when applied as a tie-breaking rule in selection, can quickly spread the population out over the $\jump$ plateau in order to get enough diversity for crossover to combine the correct solution components to escape the set of local optima.

\subsection*{Acknowledgements}
The research leading to these results has received funding from the
European Union Seventh Framework Programme (FP7/2007-2013) under grant
agreement no.\ 618091 (SAGE) and from the EPSRC under grant no.\ EP/M004252/1.
This research benefitted from Dagstuhl seminar 16011 ``Evolution and Computing'' and is based upon work from COST Action CA15140 `Improving Applicability of Nature-Inspired Optimisation by Joining Theory and Practice (ImAppNIO)' supported by COST (European Cooperation in Science and Technology).

\printbibliography

\end{document}

%% file: preamble.tex
\usepackage[utf8]{inputenc}
\usepackage{csquotes}
\usepackage[british]{babel}

\usepackage{amssymb,amsmath,amsthm}
\usepackage{mathrsfs,dsfont}

\theoremstyle{plain}
\newtheorem{theorem}{Theorem}
\newtheorem{lemma}[theorem]{Lemma}

\newtheorem{condition}[theorem]{Condition}

\def\Re{\mathds{R}}
\def\Na{\mathds{N}}

\newcommand{\prob}[1]{\Pr\left(#1\right)} %

\def\E{\ensuremath{\mathrm{E}}}

\newcommand{\dist}{d_{\mathrm{H}}}

\usepackage[algo2e,ruled,vlined,linesnumbered]{algorithm2e}
\usepackage{xspace,xcolor}
\usepackage{balance}
\usepackage{hyperref}
\usepackage{booktabs,multirow}
\usepackage{pdflscape}
\usepackage{enumitem}
\usepackage{authblk}

\usepackage[style=ieee,natbib=true,backend=bibtex8,sorting=nyt,
            isbn=false,url=false,doi=false,eprint=false]{biblatex}
\allowdisplaybreaks[1]
\newcommand{\ab}{\hspace{0.125em}} %
\newcommand{\ie}{\hbox{i.\ab e.}\xspace,\xspace} %
\newcommand{\eg}{\hbox{e.\ab g.}\xspace,\xspace}  %

\newcommand{\om}{\mathrm{OneMax}\xspace}

\newcommand*{\jump}{\mathrm{Jump}}
\newcommand*{\jumpK}{\jump_k}
\newcommand*{\oneoneea}{(1+1)~EA\xspace}
\newcommand*{\oneoneeas}{(1+1)~EAs\xspace}
\newcommand*{\muEA}{($\mu$+1)~EA\xspace}
\newcommand*{\muGA}{($\mu$+1)~GA\xspace}
\newcommand*{\mut}{\mathrm{mutate}}
\newcommand*{\XO}{\mathrm{crossover}}
\newcommand{\pc}{p_{\mathrm{c}}}
\newcommand*{\bigO}{\mathrm{O}}
\newcommand*{\bigOh}[1]{\mathord{\bigO}\mathord{\left(#1\right)}}
\newcommand*{\smallO}{\mathrm{o}}
\newcommand*{\smallOh}[1]{\mathord{\smallO}\mathord{\left(#1\right)}}
\newcommand*{\ThetaN}[1]{\Theta\mathord{\left(#1\right)}}
\newcommand*{\poly}[1]{\mathord{\mathrm{poly}}\mathord{\left(#1\right)}}
\newcommand*{\eulerE}{\mathrm{e}}
\newcommand*{\speciesConstant}{c}
\newcommand*{\varpm}{p_{\mathrm{m}}}

\newcommand{\ignore}[1]{}

\usepackage{tikz}
\usepackage{pgfplots}
\usepackage{pgfplotstable}

%% file: results-table.tex
\ifarxiv
\begin{landscape}
\setlength{\extrarowheight}{5pt}
\fi
\begin{table*}
\small
\centering
\ifarxiv\else\resizebox{\linewidth}{!}{\fi
    \begin{tabular}{@{}lrrrrrrr}\toprule
      \multirow{2}{*}{\textbf{Mechanism}}
    & \multirow{2}{*}{$\boldsymbol{\mu}$}
    & \multirow{2}{*}{$\boldsymbol{\pc}$}
    & \multirow{2}{*}{$\boldsymbol{\varpm}$}
    & \multicolumn{3}{@{}c}{\textbf{Problem}}\\
    \cmidrule{5-7}
    &&&&
      $\boldsymbol{k=2}$
    & $\boldsymbol{k=4}$
    & $\textbf{any }\boldsymbol{k}$\\
    \midrule

    \multirow{2}{*}{\textbf{No mechanism}, Thm.~\ref{the:upper-bound-no-mechanism}}
    & $\ThetaN{\frac{\sqrt{n}}{\sqrt{\log{n}}}}$
    & $1$
    & $1/n$
    & $\bigOh{\frac{n^{1.5}}{\sqrt{\log{n}}}}$
    & $\bigOh{\frac{n^{3.5}}{\sqrt{\log{n}}}}$
    & $\bigOh{\frac{n^{k-0.5}}{\sqrt{\log{n}}}}$\\

    & $\ThetaN{n}$
    & $1$
    & $1/n$
    & $\bigOh{n^2\log{n}}$
    & $\bigOh{n^3\log{n}}$
    & $\bigOh{n^{k-1}\log{n}}$\\
    \midrule

    \textbf{No mechanism}, Thm.~\ref{thm:highmutationruntime}
    & $\ThetaN{\log{n}}$
    & $1$
    & $(1+\delta)/n$
    & $\bigOh{n\log{n}\log\log{n}}$
    & $\bigOh{n^3}$
    & $\bigOh{n\sqrt{k}\log{n}\log\log{n} + n^{k-1}}$\\
    \midrule

    \multirow{2}{*}{\textbf{Duplicate elim.}, Thm.~\ref{thm:duplicate-elimination}}
    & $2$
    & $\Omega(1)$
    & $1/n$
    & $\bigOh{n\log n}$
    & $\bigOh{n^3}$
    & $\bigOh{n \sqrt{k}\log n + n^{k - 1}}$ \\

    & $2$
    & $1 - \Omega(1)$
    & $1/n$
    & $\bigOh{n\log n}$
    & $\bigOh{n^3}$
    & $\bigOh{n \log n + n^{k - 1}}$ \\
    \midrule

    \multirow{2}{*}{\textbf{Duplicate min.}, Thm.~\ref{thm:duplicate-minimization}}
    & $2$
    & $\Omega(1)$
    & $1/n$
    & $\bigOh{n\log n}$
    & $\bigOh{n^3}$
    & $\bigOh{n \sqrt{k}\log n + n^{k - 1}}$ \\

    & $2$
    & $1 - \Omega(1)$
    & $1/n$
    & $\bigOh{n\log n}$
    & $\bigOh{n^3}$
    & $\bigOh{n \log n + n^{k - 1}}$ \\
    \midrule

    \textbf{Det. crowding}, Thm.~\ref{thm:deterministicCrowding}
    & $2$
    & $k/n$
    & $1/n$
    & $\bigOh{n\log n}$
    & $\bigOh{n\log n}$
    & $\bigOh{n\log n + n\eulerE^{5k}}$ \\
    \midrule

    \textbf{Convex hull max.}, Thm.~\ref{thm:convexHull}
    & $2$
    & $1 - \Omega(1)$
    & $1/n$
    & $\bigOh{n^2\log n}$
    & $\bigOh{n^2\log n}$
    & $\bigOh{n^2\log n + 4^k}$ \\
    \midrule

    \textbf{Hamming dist. max.}, Thm.~\ref{thm:total_hd}
    & $2$
    & $1 - \Omega(1)$
    & $1/n$
    & $\bigOh{n\log n}$
    & $\bigOh{n\log n}$
    & $\bigOh{n\log n + nk\log k + 4^k}$ \\
    \midrule

    \textbf{Fitness sharing}, Thm.~\ref{thm:fitness_sharing}
    & $2$
    & $1 - \Omega(1)$
    & $1/n$
    & $\bigOh{n\log n}$
    & $\bigOh{n\log n}$
    & $\bigOh{n\log n + nk\log k + 4^k}$ \\
    \midrule

    \textbf{Island model,} Thm.~\ref{thm:islandModel}
    & $2$
    & n/a
    & $1/n$
    & $\bigOh{n\log n}$
    & $\bigOh{n\log n}$
    & $\bigOh{n\log n + kn + 4^k}$ \\

    \bottomrule
    \end{tabular}
    \ifarxiv\else}\vspace{0.5em}\fi
    \caption{\label{tab:results} Comparison of the asymptotically best known run time bounds for the \muGA with crossover rate $\pc>0$ and mutation rate $\varpm$ on $\jumpK$ with bit string length $n$. More general results are presented in the respective theorems. Note that some of the results are subject to further conditions, \eg $k = o(n)$.}
\end{table*}
\ifarxiv
\end{landscape}
\fi

%% file: preliminaries.tex
The $\jumpK \colon \{0,1\}^n \to \Na$ class of pseudo-Boolean fitness
functions  was originally introduced
by \citet{Jansen2002Crossover}. The function value increases with the number of
1-bits in the bit string until a \emph{plateau} of local optima is reached,
consisting of all points with $n - k$ 1-bits. However, its only global optimum
is the all-ones string $1^n$\!.
Between the plateau and the global optimum, there is a valley of deteriorated
fitness, which we call the \emph{gap} of length $k$, and the algorithm has to
jump over this gap to optimise the function.

The function is formally defined as
\[
    \jumpK(x) =
    \begin{cases}
        k + |x|_1  &\textrm{if } |x|_1 = n \textrm{ or } |x|_1 \leq n - k,\\
        n - |x|_1	&\textrm{otherwise,}
    \end{cases}
\]
where $|x|_1 = \sum_{i = 1}^n x_i$ is the number of 1-bits in $x$.

We will analyse the performance of a standard steady-state
\muGA~\cite{Jansen2002Crossover} using uniform crossover (i.e., each bit of the
offspring is chosen uniformly at random from one of the parents) and standard
bit mutation (i.e., each bit is flipped with probability $\varpm$). The
algorithm uses a population of $\mu$ individuals. In each generation, a new
individual is created. With probability $\pc$, it is created by selecting two
parents from the population uniformly at random, crossing them over, and then
applying mutation to the resulting offspring. With probability $1- \pc$
instead, one single individual is selected and only mutation is applied. The
generation is concluded by removing the worst individual from the population and
breaking ties uniformly at random.
Algorithm~\ref{alg:mu+1-GA} shows the pseudocode for the \muGA{}.
\begin{algorithm2e}
    \caption{\muGA
    }
    \label{alg:mu+1-GA}
    $P \gets \mu \textrm{ individuals, uniformly at random from } \{0, 1\}^n$\;
    \While{$1^n \notin P$}
    {
        Choose $p \in [0, 1]$ uniformly at random\;
        \eIf{$p \leq \pc$}
        {
            Choose $x, y \in P$ uniformly at random\;
            $z \gets \mut(\XO(x, y))$\;
        }
        {
            Choose $x \in P$ uniformly at random\;
            $z \gets \mut(x,\varpm)$\;
        }
        $P \gets P \cup \{z\}$\;
        Remove one element from $P$ with lowest fitness, %
        breaking ties according to the tie-breaking rule, and remove it from $P$\;\label{li:tie-breaking}
    }
\end{algorithm2e}

Diversity-preserving mechanisms, which are introduced in the tie-breaking rule in line
\ref{li:tie-breaking} of the algorithm, can be of four varieties: (1) methods that encourage
\emph{genotypic uniqueness}, (2) methods that increase \emph{genotypic
  distance}, (3) methods that employ parallelism, and (4) no specific diversity-preserving
  mechanism, which corresponds to tie-breaking completely uniform at random.

In the first variety, tie-breaking mechanisms are used to promote
genotypic uniqueness in the population. The method of
\textbf{duplicate elimination} breaks ties in the selection process by
making sure an offspring replaces a duplicated string out of the
least-fit individuals. The advantage is that the crossover operator
does not have to wait to get two different individuals. Very similar
to duplicate elimination is \textbf{duplicate minimisation.}  In this
case, ties are broken such that the least-fit individual that has the
highest number of duplicates is removed.  This lowers the time to get
many pairs of different individuals. \textbf{Deterministic crowding}
always chooses a parent of the current offspring for removal in the
selection. This allows to get different individuals more
efficiently. Note that our result for deterministic crowding is the
only analysis which does not extend up to $\pc=\Theta(1)$. Recall that
deterministic crowding is here only applied in the tie-breaking rule
and not in every generation as in other implementations of deterministic
crowding \cite{Lehre2011Crossover,OlivetoHeYao2008}.

The second variety of diversity-preserving mechanisms we investigate
are based on increasing genotypic distance between individuals in the
population. A geometric crossover tends to reduce the size of the
convex hull of the population~\cite{MoraglioS:c:12}. It is therefore
natural to break ties via a \textbf{convex hull maximisation.}
Maximising the convex hull can be simulated by \textbf{maximising the
  total Hamming distance} among the population.

The third and final variety of mechanisms we study is based on relying
on parallelism to obtain and preserve diversity. In an \textbf{island
  model,} diversity is attained by keeping parts of the population on
separate processors. We give an upper bound on the expected run time
for the the single-receiver island model~\cite{Watson2007,Neumann2011}.

The most interesting behaviour
of the population
presented in this paper occurs after the entire population is stuck at local optima, the so-called plateau.
That is because under the right condition the population diversity will emerge during this stage.
Then after sufficient progress is made in diversity, crossover and
mutation can work together on the plateau to create an optimal
solution in $\smallOh{n^k}$ time. This is captured by
Lemma~\ref{lem:jump-to-opt} which will be presented later in the paper.

%% file: time-to-plateau.tex
We make the distinction between two cases $\pc = \Omega(1)$, which will be covered by Lemma~\ref{lem:time-to-plateau}, and $\pc = 1 - \Omega(1)$, which will be covered by Lemma~\ref{lem:plateau-arrival}.
Note that when both $k$ and $\mu$ are constants independent of $n$, the two cases provide the same asymptotic expected time to reach the plateau.

For $\pc = \Omega(1)$, we rely on the steps that crossover occurs and make use
of the following general result, which provides an upper bound on the expected
time for the \muGA to reach some region $A_m$ of the search space. Here we
consider a \emph{fitness-based} partition (see \cite{Jansen2013} for a formal
definition) into levels $(A_i)_{i\in[m]}$ (thus, $A_m$ is the last level) and define
$A_{\geq j} := \bigcup_{i=j}^m A_i$.
\begin{theorem}\label{thm:mu-plus-one-levels}
  Let $(A_i)_{i\in[m]}$ be a fitness-based partition of the search space
  into $m\in\Na$ levels. If there exist parameters
  $\varepsilon,s_1,\ldots, s_{m-1}\in(0,1]$
  such that for all $j\in[m-1]$
  \begin{enumerate}
  \item $\min_{x\in A_{\geq j}, y\in A_{\geq j+1}}\\\prob{\mut(\XO(x,y))\in
         A_{\geq j+1}}\geq \varepsilon$ and
  \item $\min_{x,y\in A_{j}}\prob{\mut(\XO(x,y))\in
         A_{\geq j+1}}\geq s_j$ %
  \end{enumerate}
  then the expected number of iterations until the entire population
  of the \muGA with $\pc = \Omega(1)$ is in %
  $A_m$ is
  $\bigOh{(\mu m/\varepsilon)\log(\mu)+\sum_{j=1}^{m-1}1/s_j}$.
\end{theorem}
\begin{proof}
  The proof follows \cite{Corus2014}, but we avoid a
  detailed drift analysis because the algorithm is elitist. Let the
  \emph{current level} be the smallest $j\in[m]$ such that the population
  contains less than $\mu/2$ individuals in $A_{\geq j+1}$. By
  definition, there are at least $\mu/2$ individuals in $A_{\geq j}$,
  where $j$ is the current level.

  Since the algorithm is elitist, the number of individuals in
  $A_{\geq j}$ is non-decreasing for any $j\in[m]$. For an upper
  bound, we ignore any improvements where mutation only is used (i.e.,
  lines 8 and 9 in Alg.~\ref{alg:mu+1-GA}).

  Assume that there are $i$ individuals in $A_{\geq j+1}$, hence $0\leq i < \mu/2$.
  If $i=0$, then an individual in $A_{\geq j+1}$ can be created by
  selecting two individuals from $A_j$, crossing them over, and
  mutating them such that the offspring is in $A_{\geq j+1}$ and an
  individual not in $A_{\geq j+1}$ is removed.
  The probability of this event is $\pc s_j/4$.

  If $0<i<\mu/2$, then the number of individuals in $A_{\geq j+1}$ can
  be increased by selecting an individual in $A_{\geq j}$ and an
  individual in $A_{\geq j+1}$, crossing them over, and mutating them
  such that the offspring is in $A_{\geq j+1}$ and one of the
  $\mu-x>\mu/2$ individuals not
  in $A_{\geq j+1}$ is removed.  This event occurs with probability
  at least $(\pc/2)(i/\mu)\varepsilon.$

  The expected time to increase the number of individuals in
  $A_{\geq j+1}$ from 0 to $\mu/2$, i.e., to increase the current level
  by at least one, is $4/(\pc s_j)+2\mu/(\pc \varepsilon)\sum_{i=1}^{\mu/2}1/i$.
  Hence, the expected time until at least half of the population is in
  $A_m$ is $\bigOh{(\mu
  m/\varepsilon)\log(\mu)+\sum_{j=1}^{m-1}1/s_j}$.

  We now consider the time to remove individuals from the lowest
  fitness level in the population. Assume that there are $0<i'<\mu/2$
  individuals in the lowest level $j<m$. The number of individuals in
  level $j$ can be reduced by crossing over an individual in level $j$
  and one of the at least $\mu/2$ individuals in level $m$, and
  mutating the offspring so that it belongs to $A_{\geq j+1}$. By
  condition 1, this event occurs with probability at least
  $\pc(\varepsilon/2)(i'/\mu)$. Hence, the expected time to remove all
  individuals from the lowest level $j$ is no more than
  $(2/\pc \varepsilon)\mu\sum_{i'=1}^{\mu/2}
  1/i'=\bigOh{(\mu/\varepsilon)\log\mu}$. The expected time until all
  individuals in fitness levels lower than $m$ have been removed is
  therefore $\bigOh{\mu(m/\varepsilon)\log\mu}$.
\end{proof}

We apply Theorem \ref{thm:mu-plus-one-levels} to bound the time until
the entire population reaches the plateau.

\begin{lemma}\label{lem:time-to-plateau}%
  Consider the \muGA optimizing $\jumpK$ with $\pc = \Omega(1)$ and
  $\varpm=\Theta(1/n)$.
  Then the expected time until either the optimum has been found
  or the entire population is on the plateau is
  $\bigOh{n\sqrt{k}(\mu\log\mu+\log n)}$.
\end{lemma}
\begin{proof}
  We apply Theorem \ref{thm:mu-plus-one-levels}, and
  divide the search space into $m:=n$ levels with the partition
  \begin{align*}
    A_j & :=
          \begin{cases}
            \{x\in\{0,1\}^n\mid   |x|_1=n-j\} & \text{if } 1\leq j<k,\\
            \{x\in\{0,1\}^n\mid   |x|_1=j-k\} & \text{if } k\leq j< n,\\
            \{x\in\{0,1\}^n\mid   |x|_1\in \{n-k, k\}\} & \text{if } j=n,
          \end{cases}
  \end{align*}
  where $|x|_1:=\sum_{i=1}^n x_i$ is the number of 1-bits in bit string $x$.

  We call any search point $x\in\{0,1\}^n$ with $n-k<|x|_1<n$ a
  \emph{gap-individual}.
  We first claim that the probability of producing a gap-individual
  by crossing over two
  individuals $x\in A_{\geq j}$ and $y\in A_{\geq j+1}$ with $k\leq j<n+1$
  is bounded from above by
  \begin{align}
    \prob{n-k<|\XO(x,y)|_1<n} < \frac{1}{2} - \frac{1}{4\sqrt{k}}\ .
  \end{align}

  To see why this claim holds, we first argue that the probability of
  producing a gap-individual is highest when the parents $x$ and $y$
  have $n-k$ and $n-k-1$ 1-bits, respectively. More formally, obtain
  $x'$ by flipping an arbitrary 0-bit in $x$, and $y'$ by flipping an
  arbitrary 0-bit in $y$. Then, we have the stochastic dominance
  relationships $|\XO(x,y)|_1\preceq |\XO(x',y)|_1$ and
  $|\XO(x,y)|_1\preceq |\XO(x,y')|_1$. By repeating this argument, we
  obtain $|\XO(x,y)|_1\preceq |\XO(x'',y'')|_1$ for two bit strings
  $x''$ and $y''$ with $|x''|_1=n-k$ and $|y''|_1=n-k-1$. The
  probability of obtaining a search point with exactly $k$ 0-bits when
  crossing over two bit strings with $k$ and $k+1$ 0-bits is minimised
  when the positions of the 0-bits in the two bit strings differ.
  Hence, for bit strings $x''$ and $y''$, we have the lower bound
  \begin{align*}
    \prob{|\XO(x'',y'')|_1=n-k}
    & >    {2k+1\choose k}\cdot 2^{-(2k+1)}\\
    &  \geq \frac{4^k}{2\sqrt{k}}\cdot 2^{-(2k+1)}
      =    \frac{1}{4\sqrt{k}}\ .
  \end{align*}
  The crossover operation between $x''$ and $y''$ produces two offspring
  $u''$ and $v''$ where $|x''|_1+|y''|_1=|u''|_1+|v''|_1,$ and returns $u''$ or $v''$
  with equal probability. Without loss of generality, assume
  that $|u''|_1\geq |v''|_1$. We must have $|v''|<n-k\leq |u''|_1$ because
  \begin{align}
    2(n-k)-1 = |x''|_1+|y''|_1 = |u''|_1+|v''|_1 \leq 2|u''|_1\ ,\label{eq:improvecrossover}%
  \end{align}
  and similarly
  \begin{align*}
    2(n-k)-1 = |u''|_1+|v''|_1 \geq 2|v''|_1\ ,
  \end{align*}
  hence
  \begin{align}
    \prob{|\XO(x'',y'')|_1\geq n-k} = \frac{1}{2}\ .
  \end{align}
  Our claim now follows, because
  \begin{align*}
    &\prob{n-k<|\XO(x,y)|_1<n} \\
    & <     \prob{n-k< |\XO(x,y)|_1}\\
    &  \leq  \prob{n-k< |\XO(x'',y'')|_1}\\
    & =     \prob{n-k\leq |\XO(x'',y'')|_1}\\
    &\quad    - \prob{|\XO(x'',y'')|_1=n-k}\\
    & \leq \frac{1}{2}-\frac{1}{4\sqrt{k}}\ .
  \end{align*}

  We now show that condition 1 of Theorem \ref{thm:mu-plus-one-levels}
  holds for the parameter
  $\varepsilon:=(1-\varpm)^n/(4\sqrt{k})=\Theta(1/\sqrt{k})$.
  As before, assume that $x\in A_{\geq j}$ and $y\in A_{\geq j+1}$ for
  $j\geq k$. As argued in (\ref{eq:improvecrossover}), a crossover between
  $x$ and $y$ produces two offspring $u$ and $v$ where $|u|_1\geq j+1$.
  The offspring therefore satisfies
  \begin{align*}
    & \prob{\XO(x,y)\in A_{\geq j+1}} \\
    &=
    \prob{j+1\leq |\XO(x,y)|_1} \\
    &\quad - \prob{n-k<|\XO(x,y)|_1<n}
     \geq
    \frac{1}{4\sqrt{k}}\ .
  \end{align*}
  Finally, with probability $(1-\varpm)^{n}$\!,
  none of the bits are flipped during mutation, which implies
  \begin{align*}
    \prob{\mut(\XO(x,y))\in A_{\geq j+1}} \geq \varepsilon\ .
  \end{align*}
  The same bound holds for levels $j<k$, except that we count 0-bits
  instead of 1-bits, and we do not need to account for the probability
  of producing gap individuals.

  We now show that condition 2 of Theorem~\ref{thm:mu-plus-one-levels}
  holds.  Assume that $x,y\in A_{j}$ for $j\geq k$. Then, following the same
  argument as above
  \begin{align*}
    \prob{\XO(x,y)\in A_{\geq j}}\geq \frac{1}{4\sqrt{k}}\ .
  \end{align*}
  The probability that the mutation operator flips at
  least one of the $n-j$ 0-bits, and no other bits, is at least
  $(n-j)\varpm(1-\varpm)^{n-1}$. Hence,we can use the parameter
  $s_j := (n-j)\varpm(1-\varpm)^{n-1}/(4\sqrt{k})=\Theta((n-j)/(n\sqrt{k}))$.
  The same bound holds for levels $j<k$, except that we count 0-bits
  instead of 1-bits, and we do not need to account for the probability
  of producing gap-individuals.

  The result now follows from Theorem \ref{thm:mu-plus-one-levels}.
\end{proof}

For $\pc= 1 - \Omega(1)$, we rely on mutation-only steps, hence, we make use of the result presented in \cite{Witt06}.

\begin{lemma}\label{lem:plateau-arrival}
Consider the \muGA optimising $\jumpK$ with $\pc = 1 - \Omega(1)$. Then the expected time until either the optimum has been found or the entire population is on the plateau is $\bigO(\mu n + n \log n + \mu \log \mu)$.
\end{lemma}
\begin{proof}
Before the population reaches the plateau, $\jumpK$ is identical to $\om$. In addition, our \muGA performs the same steps as the \muEA during generations that it does not perform crossover. Therefore, compared to the \muEA, any slow-down caused by crossover in \muGA can only contribute a constant factor to the run time, as long as $\pc = 1-\Omega(1)$. It then follows from \cite{Witt06} that the \muGA also needs $\bigOh{\mu n + n \log n}$ to reach the plateau. After that, similar to the last part of the proof of Theorem~\ref{thm:mu-plus-one-levels}, we only need an extra $\bigOh{\mu \log \mu}$ term that captures the waiting time until all individuals are on the plateau.
\end{proof}

%% file: population-dynamic.tex
Previous observations of simulations have revealed the following behaviour. Assume the algorithm has reached a population where all individuals are identical. We refer to identical individuals as a \emph{species}, hence, in this case, there is only one species. Eventually, a mutation will create a different search point on the plateau, leading to the creation of a new species. Both species may shrink or grow in size, and there is a chance that the new species disappears and we go back to one species only.

However, the existence of two species also serves as a catalyst for creating further species in the following sense. Say two parents $0001111111$ and $0010111111$ are recombined, then crossover has a good chance of creating an individual with $n-k+1$ $1$s, e.g., $0011111111$. Then mutation has a constant probability of flipping any of the $n-k-1$ unrelated 1-bits to 0, leading to a new species, e.\,g., $0011111011$. This may lead to a sudden burst of diversity in the population.

Simulations for the standard mutation rate $1/n$ indicated that the size of the largest species is an important factor for describing this diversity. Due to the ability to create new species, the size of the largest cluster performs an almost fair random walk.  Once its size has decreased significantly from its maximum $\mu$, there is a good chance for recombining two parents from different species. This helps in finding the global optimum, as crossover can increase the number of $1$s in the offspring, compared to its parents, such that fewer bits need to be flipped by mutation to reach the optimum. This is formalised in the following lemma.
\begin{lemma}
\label{lem:success-probability}
The probability that the global optimum is constructed by a uniform crossover of two parents with Hamming distance~$2d$, followed by mutation, is
\begin{align}
& \sum_{i=0}^{2d} \binom{2d}{i}
\frac{1}{2^{2d} n^{k+d-i}} \left(1 - \frac{1}{n}\right)^{n-k-d+i}\\
\ge\;& \frac{1}{2^{2d} n^{k-d}} \left(1 - \frac{1}{n}\right)^{n-k+d}.
\end{align}
\end{lemma}
\begin{proof}
  For a pair of search points on the plateau with Hamming
  distance~$2d$, both parents have $d$ $1$s among the $2d$ bits that
  differ between parents, and $n-k-d$ $1$s outside this area. Assume
  that crossover sets $i$ out of these $2d$ bits to~1, which happens
  with probability $\binom{2d}{i} \cdot 2^{-2d}$\!. Then mutation
  needs to flip the remaining $k+d-i$ $0$s to~1. The probability that
  such a pair creates the optimum is hence
\[
\sum_{i=0}^{2d} \binom{2d}{i}
\frac{1}{2^{2d} n^{k+d-i}} \left(1 - \frac{1}{n}\right)^{n-k-d+i}.
\]
The second bound is obtained by ignoring summands $i < 2d$ for the inner sum.
\end{proof}
Note that even a Hamming distance of $2$, i.e., $d=1$, leads to a probability of $\Omega(n^{-k+1})$, provided that such parents are selected for reproduction. The probability is by a factor of~$n$ larger than the probability $\Theta(n^{-k})$ of mutation without crossover reaching the optimum from the plateau.

We will show that this effect leads to a speedup of nearly $n$ for the \muGA, compared to the expected time of $\Theta(n^{k})$ for the (1+1)~EA~\cite{Droste2002} and other EAs only using mutation.

The idea behind the analysis is to investigate the random walk underlying the size of the largest species. We bound the expected time for this size to decrease to $\mu/2$ and then argue that the \muGA is likely to spend a good amount of time with a population of good diversity, where the probability of creating the optimum in every generation is $\Omega(n^{-k+1})$ due to the chance of recombining parents of Hamming distance at least~2.

In the following, we refer to $Y(t)$ as the size of the largest species in the population at time~$t$.
Define
\begin{align*}
  p_+(y) & := \prob{Y(t+1)-Y(t)=1\mid Y(t)=y}\ ,  \\
  p_-(y) & := \prob{Y(t+1)-Y(t)=-1\mid Y(t)=y}\ ,
\end{align*}
i.e., $p_+(y)$ is the probability that the size of the largest species
increases from $y$ to $y+1$, and $p_-(y)$ is the
probability that it decreases from
$y$ to $y-1$.

The following lemma gives bounds on these transition probabilities,
unless two parents of Hamming distance larger than~2 are selected for recombination (this case will be treated later in Lemma~\ref{lem:transition-probs-large-distance}).
We formulate the lemma for arbitrary mutation rates $\chi/n = \Theta(1/n)$ and restrict our attention to sizes $Y(t) \ge \mu/2$ as we are only interested in the expected time for the size to decrease to~$\mu/2$.
\begin{lemma}
\label{lem:transition-probs-largest-species}
For every population on the plateau of $\jump_k$ for $k =o(n)$ the following holds. Either the \muGA with mutation rate $\chi/n = \Theta(1/n)$ performs a crossover of two parents whose Hamming distance is larger than~2,
or the size~$Y(t)$ of the largest species changes according to transition probabilities $p_-(\mu) = \Omega(k/n)$ and,
for ${\mu/2 \le y < \mu}$,
\begin{align*}
 p_+(y) \le\;& \frac{y(\mu- y) (\mu+y)}{2\mu^2(\mu+1)} \left(1 - \frac{\chi}{n}\right)^n + \bigOh{\frac{(\mu-y)^2}{\mu^2 n}}\ ,\\
 p_-(y) \ge\;& \frac{y(\mu - y)(\mu+\chi y)}{2\mu^2(\mu+1)} \left(1 - \frac{\chi}{n}\right)^n\ .
\end{align*}
\end{lemma}
\begin{proof}
We call an individual belonging to the current largest species a $y$ individual
and all the others non-$y$ individuals.
In each generation, there is either no change, or one individual is
added to the population and one individual chosen uniformly at random
is removed from the population.  In order to increase the number of
$y$ individuals, it is necessary that a $y$ individual is added to
the population and a non-$y$ individual is removed from the
population. Analogously, in order to decrease the number of
$y$ individuals, it is necessary that a non-$y$ individual is added to
the population and a $y$ individual is removed from the population.

Given that $Y(t)=y$, let $p(y)$ be the probability that a
$y$ individual is created at time $t+1$, and $q(y)$
the probability that a non-$y$ individual is created.
Multiplying by the survival probabilities we have
\begin{align}
p_-(y) & = q(y)\left(\frac{y}{\mu+1}\right) \text{ and} \label{eq:pplus}\\
p_+(y) & := p(y)\left(1-\frac{y+1}{\mu+1}\right) = p(y)\left(\frac{\mu-y}{\mu+1}\right) \label{eq:pminus}\ .
\end{align}

We now estimate an upper bound on $p(y)$. We may assume that the Hamming distance between parents is at most $2$ as otherwise there is nothing to prove.
A $y$ individual can be created in the following three ways:
\begin{itemize}
\item Two $y$ individuals are selected. Crossing over two
  $y$ individuals produces another $y$ individual, which survives mutation if no bits are flipped, i.e., with
  probability $(1-\chi/n)^n$\!.
\item One $y$ individual and one non-$y$ individual are selected.  The
  crossover operator produces a $y$ individual with probability
  $1/4$, and mutation does not flip any bits with probability $(1- \chi/n)^n$\!. If the crossover operator does not
  produce a $y$ individual, then, to produce a $y$ individual, at
  least one specific bit-position must be mutated, which occurs with
  probability $\bigOh{1/n}$. The overall probability is hence
  $(1/4)(1- \chi/n)^n+\bigOh{1/n}$.
\item Two non-$y$ individuals are selected. These two individuals are either identical or have Hamming distance~$2$ (i.e., by assumption).
In the first case, they both have one of the $k$ $0$-bit positions of a $y$ individual set to $1$. In the second case, they either
both have one of the $k$ $0$-bit positions of a $y$ individual set to $1$, or they both have one of the $n-k$ $1$-bit positions set to $0$.
In both cases, crossover %
cannot change the value of such a bit. Thus, at least one
  specific bit-position must be flipped, which occurs with probability
  $\bigOh{1/n}$.
\end{itemize}
Taking into account the probabilities of the three selection events
above, the probability of producing a $y$ individual is
\begin{align*}
  p(y) & =  \left(\frac{y}{\mu}\right)^2\left(1-\frac{\chi}{n}\right)^n  + 2\left(\frac{y}{\mu }\right)\left(1-\frac{y}{\mu }\right) \cdot\\
       & \cdot \left[ \left(\frac{1}{4}\right)\left(1-\frac{\chi }{n}\right)^n +\bigOh{\frac{1}{n}}\right]
       + \frac{(\mu-y)^2}{\mu^2}\bigOh{\frac{1}{n}}\\
    & = \left(1-\frac{\chi }{n}\right)^n \left(\frac{y}{\mu}\right) \left(\frac{y}{\mu} + \frac{\mu-y}{2\mu}\right)+ \\
    & + \bigOh{\frac{y (\mu-y)}{\mu^2}  \cdot \frac{1}{n}} + \bigOh{\frac{(\mu-y)^2}{\mu^2} \cdot \frac{1}{n}}\\
    & = \frac{y(\mu+y)}{2\mu^2}\left(1-\frac{\chi }{n}\right)^n  + \bigOh{\frac{\mu-y}{\mu} \cdot \frac{1}{n}}\ . %
\end{align*}

We then estimate a lower bound on $q(y)$. In the case where $y=\mu$,
a non-$y$ individual can be added to the population if
\begin{itemize}
\item two $y$ individuals are selected and the mutation operator
  flips one of the $k$ $0$-bits and one of the $n-k$ $1$-bits. This
  event occurs with probability
  \begin{align}
   q(\mu)&=k (n-k)\left(\frac{\chi}{n}\right)^2\left(1-\frac{\chi}{n}\right)^{n-2}\\
    &= \Omega\left(\frac{k}{n} - \frac{k^2}{n^2}\right) = \Omega\left(\frac{k}{n}\right) ,
    \label{eq:q-ii}
  \end{align}
\end{itemize}
where we used that $k = o(n)$ in the last equality.

In the other case, where $y<\mu$, a non-$y$ individual can be
added to the population in the following two ways:
\begin{itemize}
\item A $y$ individual and a non-$y$ individual are selected.
Crossover produces a copy of the non-$y$ individual
with probability $1/4$, which is unchanged by mutation with
  probability $(1-\chi/n)^n$\!. %
  Secondly, with probability $1/4$, crossover
  produces an individual with $k-1$ 0-bits.
  Mutation
  then creates a non-$y$ individual by flipping a single of the $n-k$ 1-bit
  positions that do not lead to re-creating~$y$.
  Thirdly, again with probability $1/4$, crossover produces an individual with $k+1$ 0-bits and mutation then creates a non-$y$ individual by flipping a single of $k$ 1-bits that do not lead back to~$y$.
  The above three events, conditional on selecting a $y$ individual and a non-$y$ individual, lead to a total probability of
  \begin{align*}
  & \frac{1}{4} \cdot \left(1 - \frac{\chi}{n}\right)^n\\
  +\;& \frac{1}{4} \cdot (n-k) \cdot \frac{\chi}{n} \left(1 - \frac{\chi}{n}\right)^{n-1}\\
  +\;& \frac{1}{4} \cdot k \cdot \frac{\chi}{n} \left(1 - \frac{\chi}{n}\right)^{n-1}\\
  \ge\;&  \frac{\chi + 1}{4} \cdot \left(1 - \frac{\chi}{n}\right)^n.
  \end{align*}
\item Two non-$y$ individuals are selected. In the worst case, the
  selected individuals are different, hence, crossover produces an
  individual on the plateau with probability at least $1/2$, which
  mutation does not destroy with probability $(1-\chi/n)^n$\!.
\end{itemize}
Assuming that $\mu/2 \le y<\mu$ and $n$ is sufficiently large,
the probability of adding a non-$y$ individual is
\begin{align*}
  q(y) & \geq
       2\left(\frac{y}{\mu}\right)\left(1-\frac{y}{\mu }\right)  \cdot \frac{\chi +1}{4} \left(1-\frac{\chi}{n}\right)^n\\
    &   + \frac{1}{2}\left(1-\frac{y}{\mu }\right)^2 \left(1-\frac{\chi}{n}\right)^n \\
    & = \frac{(\mu-y) (\mu + \chi y)}{2 \mu ^2} \left(1-\frac{\chi}{n}\right)^n\ .
\end{align*}
Plugging $p(y)$ and $q(y)$ into equations (\ref{eq:pplus}) and (\ref{eq:pminus}), we get
\begin{align*}
p_-(y)
& \geq \left[ \frac{(\mu-y) (\mu + \chi y)}{2 \mu ^2} \left(1-\frac{\chi}{n}\right)^n\right]
      \left(\frac{y}{\mu+1}\right)\\
    &=    \frac{(\mu - y) (\mu + \chi y)y}{2\mu^2(\mu+1)} \left(1-\frac{\chi}{n}\right)^n\ .
    \end{align*}
   And we also have 
\begin{align*}
p_+(y) & =  \left[\frac{y(\mu+y)}{2\mu^2}\left(1-\frac{\chi}{n}\right)^n \!\!\! +  \bigOh{\frac{\mu-y}{\mu} \cdot \frac{1}{n}}\right]
 \left(\frac{\mu-y}{\mu+1}\right)\\
  & =  \frac{(\mu^2 - y^2)y}{2\mu^2(\mu+1)} \left(1-\frac{\chi}{n}\right)^n + \bigOh{\frac{(\mu-y)^2}{\mu^2n}}\ .\qedhere
\end{align*}
\end{proof}

Steps where crossover recombines two parents with larger Hamming distance were excluded from Lemma~\ref{lem:transition-probs-largest-species} as they require different arguments. The following lemma shows that conditional transition probabilities in this case are favourable in that the size of the largest species is more likely to decrease than to increase.
\begin{lemma}
\label{lem:transition-probs-large-distance}
  Assume that $y\geq \mu/2$ and that the \muGA on $\jumpK$ with $k=o(n)$
  and mutation rate
  $\chi/n=\Theta(1/n)$ selects two individuals on the plateau
  with Hamming distance larger than 2, then for conditional transition probabilities $p^*_-(y)$ and $p^*_+(y)$ for decreasing or increasing the size of the largest species,
  $p^*_-(y)\geq 2 p^*_+(y)$.
\end{lemma}
\begin{proof}
  Assume that the population contains two individuals $x$ and $z$ with
  Hamming distance $2\ell \leq 2 k$, where $\ell\geq 2$. Without loss of
  generality, let us assume that they differ in the first $2\ell$ bit
  positions.

In the case that the majority individual $y$ has $\ell$ 0-bits in the
first $2\ell$ positions, a $y$ individual may be produced by creating the $\ell$ 0-bits and $\ell$ 1-bits in the exact positions by crossover
and no followed mutation. Alternatively, at least $1$ exact bit has to be flipped by mutation.
Then the probability of producing a $y$ individual
from $x$ and $z$ and replacing a non $y$ individual with $y$ is less than
\begin{multline*}
p^*_+(y)  \leq \left[ \left(\frac{1}{2}\right)^{2\ell}\left(1-\frac{\chi}{n}\right)^{n} + \bigOh{\frac{1}{n}} \right] \left(\frac{\mu-y}{\mu}\right)
\\   \leq  \left(\frac{1}{2}\right)^{2\ell+1}\left(1-\frac{\chi}{n}\right)^{n} + \bigOh{\frac{1}{n}}\ .
\end{multline*}
On the other hand, the probability of producing an individual on the
plateau different from $y$  and replacing a $y$ individual is at
least
\begin{multline*}
 p^*_-(y) \geq \left({2\ell\choose \ell}-1\right)\left(\frac{1}{2}\right)^{2\ell}\left(1-\frac{\chi}{n}\right)^{n}\left(\frac{y}{\mu}\right)
\\  \geq 3\left(\frac{1}{2}\right)^{2\ell+1}\left(1-\frac{\chi}{n}\right)^{n} \geq 2 p^*_+(y)
\end{multline*}
for sufficiently large $n$.

In the other case, assume that the majority individual $y$ does not
have $\ell$ 0-bits in the first $2\ell$ bit-positions. Then the
mutation operator must flip at least one specific bit among the last
$n-2\ell$ positions to produce $y$, which occurs with probability
$\bigOh{1/n}$, while the probability to produce a non-$y$ individual on the
plateau is still $\Omega(1)$.
\end{proof}

%% file: standard-mutation.tex
We first analyse the \muGA with the standard mutation rate of $1/n$, i.e., $\chi=1$.
We show that the diversity emerging in the \muGA leads to a speedup of nearly $n$ for the \muGA, compared to the expected time of $\Theta(n^{k})$ for the (1+1)~EA~\cite{Droste2002} and other EAs only using mutation.
\begin{theorem}
\label{the:upper-bound-no-mechanism}
The expected optimisation time of the \muGA with $\pc=1$ and $\mu \le \kappa n$, for some constant~$\kappa > 0$, on $\jump_k$, $k = o(n)$, is
\[
\bigOh{\mu n \sqrt{k} \log(\mu) + n^k/\mu + n^{k-1}\log(\mu)}\ .
\]
\end{theorem}
For $k \ge 3$, the best speedup is of order $\Omega(n/\log n)$ for $\mu = \kappa n$.
For $k=2$, the best speedup is of order $\Omega(\sqrt{n/\log n})$ for $\mu = \Theta(\sqrt{n/\log n})$.

Note that for mutation rate $1/n$, the dominant terms in Lemma~\ref{lem:transition-probs-largest-species} are equal, hence the size of the largest species performs a fair random walk up to a bias resulting from small-order terms. This confirms our intuition from observing simulations.
The following lemma formalises this fact: in steps where the size $Y(t)$ of the largest species changes, it performs an almost fair random walk.
\begin{lemma}
\label{lem:conditional-transition-probabilities}
For the random walk induced by the size of the largest species, conditional on the current size~$y$ changing, for $\mu/2 < y < \mu$, the probability of increasing~$y$ is at most $1/2 + \bigOh{1/n}$, and the probability of decreasing it is at least $1/2 - \bigOh{1/n}$.
\end{lemma}
\begin{proof}
We only have to estimate the conditional probability of increasing~$y$ as the two probabilities sum up to~1. The sought probability is given by
$p_+(y)/(p_+(y) + p_-(y))$, which is strictly increasing in $p_+(y)$.
Lemma~\ref{lem:transition-probs-large-distance} states that whenever the \muGA recombines two parents of Hamming distance larger than~2, the claim on conditional probabilities follows. Hence we assume in the following that this does not happen.

Using the lower bound for $p_+(y)$ and the upper bound from $p_-(y)$ from Lemma~\ref{lem:transition-probs-largest-species}, with implicit constant $c_+$ in the asymptotic term for $p_+$, we get
\begin{align*}
\frac{p_+(y)}{p_+(y) + p_-(y)}
\le\;& \frac{\frac{y(\mu+y)(\mu-y)}{2\mu^2(\mu+1)} \cdot \left(1 - \frac{1}{n}\right)^n + \frac{c_+(\mu-y)^2}{\mu^2 n}}{\frac{y(\mu+y)(\mu-y)}{\mu^2(\mu+1)} \cdot \left(1 - \frac{1}{n}\right)^n + \frac{c_+(\mu-y)^2}{\mu^2 n}}\\
=\;& \frac{1}{2} + \frac{\frac{c_+(\mu-y)^2}{2\mu^2 n}}{\frac{y(\mu+y)(\mu-y)}{\mu^2(\mu+1)} \cdot \left(1 - \frac{1}{n}\right)^n + \frac{c_+(\mu-y)^2}{\mu^2 n}}\\
=\;& \frac{1}{2} + \frac{\frac{c_+(\mu-y)}{2\mu n}}{\frac{y(\mu+y)}{\mu(\mu+1)} \cdot \left(1 - \frac{1}{n}\right)^n + \frac{c_+(\mu-y)}{\mu n}}
\end{align*}
where in the last step we multiplied the last fraction by $\mu/(\mu-y)$.
Now the numerator is $\bigOh{1/n}$. Since $\mu/2 < y < \mu$, we have $\frac{y(\mu+y)}{\mu(\mu+1)} = \Theta(1)$. Along with $\left(1 - \frac{1}{n}\right)^n = \Theta(1)$ and $\frac{c_+(\mu-y)}{\mu n} = \bigOh{1/n}$, the denominator simplifies to $\Theta(1) + \bigOh{1/n} = \Theta(1)$. Hence the last fraction is $\bigOh{1/n}$, proving the claim.
\end{proof}

We use these transition probabilities to bound the expected time for the random walk to hit $\mu/2$.
\begin{lemma}
\label{lem:hitting-time}
Consider the random walk of $Y(t)$, starting in state~$X_0 \ge \mu/2$.
Let $T$ be the first hitting time of state~$\mu/2$. If $\mu = \bigOh{n}$, then $\E{(T \mid X_0)} = \bigOh{\mu n + \mu^2 \log \mu}$ regardless of $X_0$.
\end{lemma}
\begin{proof}
Let $E_i$ abbreviate $\E(T \mid X_0 = i)$, then $E_{\mu/2} = 0$ and
$E_\mu = \bigOh{n} + E_{\mu-1}$ as $p_-(\mu) = \Omega(1/n)$ by Lemma~\ref{lem:transition-probs-largest-species}.

For $\mu/2 < y < \mu$, the probability of leaving state~$y$ is always (regardless of Hamming distances between species) bounded from below by the probability of selecting two $y$ individuals as parents, not flipping any bits during mutation, and choosing a non-$y$ individual for replacement (cf.\ Lemma~\ref{lem:transition-probs-largest-species}):
\begin{align*}
p_+(y) + p_-(y) \ge\;&
\frac{y^2}{\mu^2} \cdot \left(1-\frac{1}{n}\right)^n \cdot \frac{\mu-y}{\mu+1}
\ge \frac{\mu-y}{24\mu}\ ,
\end{align*}
as $y \ge \mu/2$, $\mu+1\le 3\mu/2$ (since $\mu \ge 2$), and $(1-1/n)^n \ge 1/4$ for $n \ge 2$.
Using conditional transition probabilities $1/2 \pm \delta$ for $\delta = \bigOh{1/n}$ according to Lemma~\ref{lem:conditional-transition-probabilities}, $E_i$ is bounded as
\[
E_i \le \frac{24\mu}{\mu-i} + \left(\frac{1}{2} - \delta\right) E_{i-1} + \left(\frac{1}{2} + \delta\right) E_{i+1}\ .
\]
This is equivalent to
\[
\left(\frac{1}{2} - \delta\right) \cdot (E_i - E_{i-1}) \le \frac{24\mu}{\mu-i} + \left(\frac{1}{2} + \delta\right) \cdot (E_{i+1} - E_{i})\ .
\]
Introducing $D_i := E_i - E_{i-1}$, this is
\[
\left(\frac{1}{2} - \delta\right) \cdot D_i \le \frac{24\mu}{\mu-i} + \left(\frac{1}{2} + \delta\right) \cdot D_{i+1}
\]
and equivalently
\begin{align*}
D_i \le\;& \frac{\frac{24\mu}{\mu-i} + \left(\frac{1}{2} + \delta\right) \cdot D_{i+1}}{\frac{1}{2} - \delta}
\le \frac{50\mu}{\mu-i} + \alpha \cdot D_{i+1}
\end{align*}
for $\alpha := \frac{1+2\delta}{1-2\delta} = 1 + \bigOh{1/n}$, assuming $n$ is large enough.
From $E_{\mu} = \bigOh{n} + E_{\mu-1}$, we get $D_{\mu} = \bigOh{n}$, hence an induction yields
\[
D_i \le \sum_{j=i}^{\mu-1} \frac{50\mu}{\mu-j} \cdot \alpha^{j-i} + \alpha^{\mu-i} \cdot \bigOh{n}\ .
\]
Combining $\alpha = 1 + \bigOh{1/n}$ and $1+x \le \eulerE^x$ for all $x \in \Re$, we have $\alpha^{\mu} \le \eulerE^{\bigOh{\mu/n}} \le \eulerE^{\bigOh{1}} = \bigOh{1}$. Bounding both $\alpha^{j-i}$ and $\alpha^{\mu-i}$ in this way, we get
\[
D_i \le \bigOh{n} + \bigOh{\mu} \cdot \sum_{j=i}^{\mu-1} \frac{1}{\mu-j} = \bigOh{n + \mu \log \mu}\ ,
\]
as the sum is equal to $\sum_{j=1}^{\mu-i} 1/j = \bigOh{\log \mu}$.

Now,
\begin{align*}
& D_{\mu/2+1} + D_{\mu/2+2} + \dots + D_i\\
=\;& (E_{\mu/2+1} - E_{\mu/2}) + (E_{\mu/2+2} - E_{\mu/2+1}) + \dots + (E_i - E_{i-1})\\
 =\;& E_i - E_{\mu/2}\\
 =\;& E_i.
\end{align*}
Hence, we get
$E_i = \sum_{k=\mu/2+1}^i D_k
\le \bigOh{\mu n + \mu^2 \log \mu}$.
\end{proof}

Now we show that when the largest species has decreased its size to~$\mu/2$ there is a good chance that the optimum will be found within the following $\Theta(\mu^2)$ generations.
\begin{lemma}
\label{lem:jump-from-interval}
Consider the \muGA with $\pc = 1$ on $\jump_k$. If the largest species has size at most $\mu/2$ and $\mu \le \kappa n$ for a small enough constant~$\kappa > 0$, the probability that during the next $c \mu^2$ generations, for some constant~$c > 0$, the global optimum is found is
$\Omega\left(\frac{1}{1 + n^{k-1}/\mu^2}\right)$.
\end{lemma}
\begin{proof}
We show that during the $c \mu^2$ generations the size of the largest species never rises above $(3/4)\mu$
with at least constant probability. Then we calculate the probability of jumping to the optimum during the phase given that this happens.

Let $X_i$, $1 \leq i \leq c\mu^2$ be random variables indicating the increase in number of individuals of the largest species at generation $i$.
We pessimistically ignore self-loops and assume that the size of the species either increases or decreases in each generation.
Using the conditional probabilities from  Lemma~\ref{lem:conditional-transition-probabilities}, we get that the expected increase in each step is
\[ 1 \cdot (1/2 + \bigOh{1/n}) - 1 \cdot (1/2 - \bigOh{1/n}) = \bigOh{1/n}\ .\]
Then the expected increase in size of the largest species at the end of the phase is
\[\E(X) = \sum_{i=1}^{c\mu^2} X_i = \sum_{i=1}^{c\mu^2} \bigOh{1/n} = (c'\mu^2)/n \leq c'\kappa \mu \leq (1/8)\mu\ ,\]
where we use that $\mu \leq \kappa n$ and $\kappa$ is chosen small enough.

By an application of Hoeffding bounds, $\Pr{(X \geq \E(X) + \lambda)} \leq \exp(- 2 \lambda^2/\sum_{i} c_i^2)$
with $\lambda= \mu/8$ and $c_i=2$,  %
we get that $\Pr{(X \geq (2/8)\mu)} \leq \exp(-c') = 1- \Omega(1)$.
We remark that the bounds also hold for any partial sum of the sequence $X_1, \dots, X_{c\mu^2}$ (\cite{AugerD11}, Chapter 1, Theorem 1.13), i.e., with probability $\Omega(1)$ the size \emph{never} exceeds $(3/4)\mu$ in the considered phase of length $c\mu^2$ generations.

While the size does not exceed $(3/4)\mu$, in every step there is a probability of at least $1/4 \cdot 3/4 = \Omega(1)$ of selecting parents from two different species, and by Lemma~\ref{lem:success-probability}, the probability of creating the optimum is $\Omega(n^{-k+1})$.

Finally, the probability that at least one successful generation occurs in a phase of $c\mu^2$ is, using  $1-(1-p)^{\lambda} \geq (\lambda p/(1+\lambda p))$ for $\lambda \in \Na$, $p \in [0, 1]$~\cite[Lemma~10]{BadkobehLehreSudholt2015}, the probability that the optimum is found in one of these steps is
\[
\hspace*{3 em} 1 - \bigg(1-\frac{1}{\Omega(n^{-k+1})}\bigg)^{c\mu^2} \geq \Omega\bigg(\frac{\mu^2 \cdot n^{-k+1}}{1+ \mu^2 \cdot n^{-k+1}}\bigg)\ .\qedhere
\]
\end{proof}

Finally, we assemble all lemmas to prove our main theorem of this section.
\begin{proof}[Proof of Theorem~\ref{the:upper-bound-no-mechanism}]
The expected time for the whole population to reach the plateau is $\bigOh{\mu n \sqrt{k} \log(\mu) + n \sqrt{k} \log n}$ by Lemma~\ref{lem:time-to-plateau}.%

Once the population is on the plateau, we wait till the largest species has decreased its size to at most~$\mu/2$.
According to Lemma~\ref{lem:hitting-time}, the time for the largest species to reach size~$\mu/2$ is $\bigOh{\mu n + \mu^2 \log \mu}$. By Lemma~\ref{lem:jump-from-interval}, the probability that in the next $c\mu^2$ steps the optimum is found is $\Omega\left(\frac{1}{1+ n^{k-1}/\mu^2}\right)$. If not, we repeat the argument. The expected number of such trials is $\bigOh{1 + n^{k-1}/\mu^2}$, and the expected length of one trial is $\bigOh{\mu n + \mu^2 \log \mu} + c\mu^2 = \bigOh{\mu n + \mu^2 \log \mu}$. The expected time for reaching the optimum from the plateau is hence at most $\bigOh{\mu n + \mu^2 \log(\mu) + n^{k}/\mu +  n^{k-1} \log(\mu)}$.

Adding up all times and subsuming terms $\mu^2 \log(\mu) = \bigOh{\mu n \sqrt{k} \log \mu}$ and $n \sqrt{k} \log n = \bigOh{n^k/\mu + n^{k-1} \log \mu}$, noting that $k=o(n)$ completes the proof.
\end{proof}

%% file: high-mutation.tex
We now consider the run time of \muGA with mutation rate
$\chi/n= (1+\delta)/n$ for an arbitrary constant $\delta>0$. The following
theorem states that in this
setting the algorithm has at least a linear speedup compared to the
($\mu$+1) EA without crossover \cite{Witt06}.  %
By assuming a
slightly higher mutation rate, we not only obtain a bound which is by a
$\log$-factor better than Theorem~\ref{the:upper-bound-no-mechanism},
but the analysis is also significantly simpler.

\begin{theorem}\label{thm:highmutationruntime}
  The \muGA with mutation rate $(1+\delta)/n$, for a constant
  $\delta>0$, and population size $\mu\geq ck\ln(n)$ for a
  sufficiently large constant $c>0$, has  for $k=o(n)$ expected optimisation time
  $\bigOh{n\sqrt{k}\mu\log(\mu)+\mu^2+n^{k-1}}$ on $\jumpK$.
\end{theorem}

We again study the random
walk corresponding to the size of the largest species on the
plateau. For mutation rate $1/n$, this is almost an unbiased random
walk. For slightly higher mutation rates, we will see that the random
walk changes to an unfair random walk where the size of the largest
species decreases by $\Omega(1/\mu)$ in expectation. Formally, our
analysis assumes the following condition.

\begin{condition}\label{cond:species-neg-drift}
  For a constant $\delta>0$ and all $y, \mu/2\leq y\leq \mu,$
  \begin{align}
    p_-(y) & \geq
           \begin{cases}
             \Omega(1/n) & \text{ if } y=\mu,\\
             \Omega(1/\mu)    & \text{ if } \mu/2\leq y< \mu,\text{ and}\\
             (1+\delta)p_+(y) & \text{ if } \mu/2\leq y< \mu.
           \end{cases}\label{eq:posdriftcond}
\end{align}
\end{condition}

The following lemma states
that it is sufficient to increase the mutation rate slightly above
$1/n$ to satisfy the diversity condition.

\begin{lemma}\label{lemma:high-mutation-yields-diversity}
  If $\chi/n\geq (1+\delta)/n$ for any constant $\delta>0$, then
  Condition \ref{cond:species-neg-drift} holds.
\end{lemma}
\begin{proof}
  The first two inequalities follow directly from
  Lemma~\ref{lem:transition-probs-largest-species} and
  Lemma~\ref{lem:transition-probs-large-distance}.
  For any
  constant $\varepsilon>0$,
  Lemma~\ref{lem:transition-probs-largest-species} implies that
  \begin{align*}
    p_+(y) & \leq
    \frac{y(\mu-y)(\mu+y)(1+\varepsilon)}{2\mu^2(\mu+1)}\left(1-\frac{\chi}{n}\right)^n\text{ and}\\
    p_-(y) & \geq
    \frac{y(\mu-y)(\mu+\chi y)(1-\varepsilon^2)}{2\mu^2(\mu+1)}\left(1-\frac{\chi}{n}\right)^n\,.
  \end{align*}
  Thus, given that $\mu/2<y<\mu$ and $\chi\geq 1+\delta$,
  \begin{align*}
    \frac{p_-(y)}{p_+(y)}
    & \geq
    \left(\frac{\mu+\chi y}{\mu+y}\right)(1-\varepsilon)
    \geq  1+\delta'
  \end{align*}
  for some constant $\delta'>0$ when $\varepsilon$ is sufficiently
  small.
\end{proof}

Given Condition \ref{cond:species-neg-drift}, a drift argument implies
that the largest species quickly decreases to half the population
size.

\begin{lemma}\label{lemma:drift-high-mut}
  If Condition \ref{cond:species-neg-drift} holds, then the expected
  time until the largest species has size at most $\mu/2$ is
  $\bigOh{\mu^2+n}$.
\end{lemma}
\begin{proof}
  Let $Y(t)$ denote the size of the largest species at time $t$.
  We consider the drift with respect to the distance function
  $h(y) := y+(n/\mu)\eulerE^{-\kappa(\mu-y)}$\!, where $\kappa:=\ln(1+\delta)$
  over the interval $y\in[\mu/2,\mu]$. The total distance is
  $h(\mu)-h(\mu/2)=\bigOh{\mu+n/\mu}$, hence, we need to prove that the
  drift of the process $h(Y(t))$ is $\Omega(1/\mu)$.

We first bound the drift of  $Y(t)$.

Case 1: $Y(t)=\mu$. Since $Y(t+1)\leq \mu$,
the drift in this case is
\begin{align*}
 \E(Y(t)-Y(t+1)\mid Y(t)=\mu) \geq 0.
\end{align*}

Case 2: $\mu/2 < Y(t) < \mu$. By
(\ref{eq:posdriftcond}), the drift in this
case is
\begin{multline*}
  \E(Y(t)-Y(t+1)\mid Y(t)=y, \mu/2 < y < \mu)\\
    = p_-(y) - p_+(y)
    > \delta p_-(y) =\Omega(1/\mu)\ .
\end{multline*}

Consider the drift according to
$g(y) := (n/\mu)e^{-\kappa(\mu-y)}$\!.

Case 1: $Y(t)=\mu$. By (\ref{eq:posdriftcond}),
\begin{multline*}
  \E(g(Y(t))-g(Y(t+1))\mid Y_t(t) = \mu) \\ %
   = \Omega(1/n)(n/\mu)(1-\eulerE^{-\kappa})
   = \Omega(1/\mu)\ .
\end{multline*}
Case 2: $\mu/2 < Y(t) < \mu$.
By (\ref{eq:posdriftcond}),
$p_+(y)\eulerE^{\kappa}\leq p_-(y)$. The drift with respect to $g$
is therefore
\begin{align*}
  & \E(g(Y(t))-g(Y(t+1))\mid \mu/2 < Y(t) < \mu)\\
  & = p_+(y)(g(y)-g(y+1))+p_-(y)(g(y)-g(y-1)) \\
  & = (n/\mu)\eulerE^{-\kappa  (\mu -y+1)}\left(\eulerE^{\kappa }-1\right) \left(p_-(y)-p_+(y) \eulerE^{\kappa }\right)
    > 0\ .
\end{align*}
The drift is always $\Omega(1/\mu)$, and the theorem follows.
\end{proof}

After the population diversity has increased sufficiently on the
plateau, an optimal solution can be produced with the right
combination of crossover and mutation. This is captured by the
following lemma.
\begin{lemma}
  \label{lem:jump-to-opt}
  Consider a population $P$ on the $\jumpK$ plateau ($f(x) = n-k$ for
  all $x \in P$). We partition $P$ into species. For any constant
  $0 < \speciesConstant < 1$, if the largest species has size at most
  $\speciesConstant \mu$, then the optimal solution is created by
  uniform crossover followed by mutation with probability
  $\Omega((\chi/n)^{k-1})$ assuming the mutation rate is 
  $\chi/n=\Theta(1/n)$.
\end{lemma}
\begin{proof}

  Since the size of the largest species is no larger than
  $\speciesConstant\mu$, the probability that two distinct parents are
  selected for crossover is $\Omega(1)$. For the remainder of the
  proof, we assume that two parents $x$ and $y$ are selected with
  $x \neq y$.

  Let $2d > 0$ denote the Hamming distance between $x$ and $y$. Then
  $x$ and $y$ have $d$ $1$s among the $2d$ bits that differ between
  parents and $n-k-d$ $1$s outside this area. Assume that crossover
  sets exactly $i$ out of these $2d$ bits to $1$, which happens with
  probability $\binom{2d}{i} 2^{-2d}\!$. Then mutation needs to flip
  the remaining $k+d-i$ $0$s to $1$. The probability of this occurring
  is
\[
\hspace*{1 em}\sum_{i=0}^{2d} \binom{2d}{i}
2^{-2d}\left(\frac{\chi}{n}\right)^{k+d-i} \left(1 - \frac{\chi}{n}\right)^{n-k-d+i}  = \Omega((\chi/n)^{k-1})\,
\]
where we bound the sum by dropping all but the last term ($i=2d$) and use
$4^{-d} \geq \tfrac{1}{4}\left(\frac{\chi}{n}\right)^{d-1}\!\!$, since $d > 0$ and we take
$n$ to be large enough.
\end{proof}

We are now in a position to complete the run time analysis of the
algorithm. By Lemma~\ref{lem:time-to-plateau} and
Lemma~\ref{lemma:drift-high-mut}, we quickly reach a diverse
population on the plateau. From this configuration, there is a
sufficiently high probability that before the diversity is lost the
algorithm has crossed over an appropriate pair of individuals and
jumped to the optimum. If the diversity is lost, we can repeat the
argument.

\begin{proof}[Proof of Theorem \ref{thm:highmutationruntime}]
  By Lemma \ref{lem:time-to-plateau}, the expected time for the entire
  population to reach the plateau is $\bigOh{n\sqrt{k}\mu\log\mu}$, and by
  Lemma~\ref{lemma:high-mutation-yields-diversity},
  Condition~\ref{cond:species-neg-drift} is satisfied.

  Assume $c'$ sufficiently large so that $\mu\geq (c'k/\delta)\ln(n)$
  implies $(1+\delta)^{\mu/4}\geq 4cn^{k-1}+1$ for a constant $c$ that
  will be determined. We consider a phase of length
  $c(\mu^2+2n^{k-1})$ iterations and define the following three
  failure events.

  The \emph{first failure} occurs if within the first $c(\mu^2+n)$
  iterations the largest species has not become smaller than $\mu/2$
  individuals. By Lemma \ref{lemma:drift-high-mut}, the expected time
  until less than $\mu/2$ individuals belong to the largest species is
  $\bigOh{\mu^2+n}$. Hence, by Markov's inequality, the probability of this
  failure is less than $1/4$ when $c$ is sufficiently large.

  The \emph{second failure} occurs if within the next $cn^{k-1}$
  iterations there exists a sub-phase which starts with $\mu/2+1$
  individuals in the the largest species and ends with the largest
  species larger than $(3/4)\mu$ without first reducing to $\mu/2$. We
  call such a sub-phase a \emph{failure}. We model the number of
  individuals in the largest species by a Gambler's ruin argument
  \cite{Feller1968ProbIntro1}, where, by (\ref{eq:posdriftcond}), the
  probability of losing an individual in the largest species is at
  least a $(1+\delta)$-factor larger than the probability of winning
  such an individual. From standard results about the Gambler's ruin
  process \cite{Feller1968ProbIntro1}, the probability that a
  sub-phase is a failure is $\delta/((1+\delta)^{\mu/4}-1)$.  By a
  union bound, the probability that any of the at most $cn^{k-1}$
  sub-phases is a failure is no more than
  $cn^{k-1}/((1+\delta)^{\mu/4}-1) < 1/4$.

The \emph{third failure} occurs if the optimum is not found during a
sub-phase of length $cn^{k-1}$ iterations where the largest species is
always smaller than $(3/4)\mu$ individuals. In this configuration, two
individuals with Hamming distance at least 2 are selected with
probability at least $(3/4)(1/4)$. By
Lemma~\ref{lem:jump-to-opt}, the probability of obtaining the
optimum from two such individuals is
$\Omega(1/n^{k-1})$. Hence, the probability of not obtaining the
optimum during the sub-phase of length $cn^{k-1}$ is
$(1-\Omega(1/n^{k-1}))^{cn^{k-1}}\leq 1/4$ for sufficiently large $c$.

By a union bound, given a sufficiently large constant $c>0$, the
probability that none of the failures occur and the optimum is found
within a phase of length $c(\mu^2+2n^{k-1})$ iterations is at least
$1/4$. Therefore, the expected number of phases until the optimum is
found is no more than 4.
\end{proof}

%% file: duplicate-elimination.tex
In this setting, we consider the \muGA using \emph{duplicate elimination} as a tie-breaking mechanism that operates as follows. When breaking ties on the lowest-fitness individual in line~\ref{li:tie-breaking} of Algorithm~\ref{alg:mu+1-GA}, if there are no duplicates among the least-fit individuals, one is chosen at random to remove. Otherwise, we always choose one so that the number of duplicated strings decreases. 

\begin{theorem}
  \label{thm:duplicate-elimination}
  Consider the \muGA using duplicate elimination to break ties
  and a population size of $\mu < k(n-k)/2$. 
  The expected optimisation time on $\jumpK$ with $k = \smallOh{n}$ is 
  $\bigOh{\mu^2 n + n\log n + n^{k-1}}$ if $\pc = 1 - \Omega(1)$
  and $\bigOh{\mu n (\sqrt{k}\log\mu + \mu) + n\sqrt{k}\log n + n^{k-1}}$ if $\pc = \Omega(1)$.
\end{theorem}
\begin{proof}
We assume all individuals in the population are at the plateau.
Let $0 < \speciesConstant < 1$ be an arbitrary constant.
We argue that after $\bigO(\mu^2 n)$ generations in expectation, there are at most $\speciesConstant \mu$ duplicates in the population. A \emph{duplicate pair} is a pair $(x,y)$ such that $x = y$. For all $z \in \{0,1\}^n$\!, we define the (possibly empty) set
$
S_2(z) = \{ y \in \{0,1\}^n : f(y) = n - k \land \dist(z,y) = 2 \},
$
where $\dist$ denotes the Hamming distance. Let $x \in P$ be conditioned on the event that every point in $P$ is contained on the plateau. Then $|S_2(x)| = k(n-k)$, so due to our bounds on $\mu$, $|S_2(x) \smallsetminus P| \geq k(n-k)/2$.

If there are duplicates in the population, a new plateau point can be generated as follows. First, select a duplicate pair $x$, $x'$ as parents and perform crossover to obtain $x''\!\!$. Obviously $x = x' = x''\!\!$. 
Now, mutation flips exactly two specific bits of $x''$ to create one of the points in $S_2(x'') \smallsetminus P$\!. We call such a mutation a \emph{novel} mutation. The probability for any novel mutation on $x''$ is at least
\[
\frac{1}{n^2}\left(1 - \frac{1}{n}\right)^{n-2}|S_2(x'') \smallsetminus P| \geq \frac{1}{n^2}\frac{k(n-k)}{2\eulerE} = \Omega(1/n)\ .
\]
As long as there are at least $\speciesConstant\mu$ duplicates, the probability of choosing a duplicate pair for parents can be bounded as follows. Consider a partition of the population at time $t$ such that each partition contains all copies of a particular string. Let $s_1 \geq s_2 \geq \cdots \geq s_a \geq 1$ denote the sequence of partition sizes. Let $\ell^* = \max\{i : s_i > 1\}$. If $\sum_{i=1}^{\ell^*} s_i \geq \speciesConstant \mu$, then the number of duplicate pairs is at least
\[
\sum_{i=1}^{\ell^*} \binom{s_i}{2} \geq \frac{1}{2} \sum_{i=1}^{\ell^*} s_i \geq \speciesConstant \mu/2\ .
\]

Therefore, under the condition that there are at least $\speciesConstant\mu$ duplicates, the probability that a duplicate pair is selected for recombination is at least $(\speciesConstant \mu/2)/\binom{\mu}{2} = \Omega(1/\mu)$.

If a novel mutation occurs in the offspring, then it will be accepted and, consequently, a duplicate will be removed from the population. The number of duplicates cannot increase. Moreover, it decreases in each generation with probability $\Omega\big(1/(\mu n)\big)$. Hence, the expected waiting time until all but $\speciesConstant \mu$ duplicates have been removed from the plateau is $\bigO(\mu^2 n)$.

After this time, we maintain the invariant that the size of the largest equivalence class cannot be higher than $\speciesConstant\mu$. Thus, in each subsequent generation, the probability of generating the optimal string is bounded from below by Lemma~\ref{lem:jump-to-opt}. The expected number of generations until the optimal string appears in the population after this point is thus $\bigO(n^{k-1})$.

The expected number of generations to reach the optimum from the plateau is then $\bigOh{\mu^2 n + n^{k-1}}$, and the results follow by taking into account the time for the population to reach the plateau in first place, \ie by Lemma \ref{lem:time-to-plateau} for $\pc=\Omega(1)$ and Lemma \ref{lem:plateau-arrival} for $\pc = 1 - \Omega(1)$ (and here by noting that $\mu\log\mu$ is first subsumed into $\bigOh{\mu n}$ and later $\bigOh{\mu^2 n}$ since $\mu = \poly{n}$).
\end{proof}

%% file: duplicate-minimization.tex
Duplicate minimisation is similar to duplicate elimination, except that, when breaking ties, we do not choose an arbitrary duplicate but an individual that has the highest number of duplicates.

\begin{theorem}
  \label{thm:duplicate-minimization}
  Consider the \muGA using duplicate minimisation to break ties, 
  and a population size of $\mu < k(n-k)/2$. The expected optimisation time on $\jumpK$ with $k = \smallOh{n}$ is $\bigO{\left( \mu n + n\log n +n^{k-1}\right)}$ if $\pc = 1 - \Omega(1)$ and $\bigO{\left( \mu n \sqrt{k} \log\mu  + n \sqrt{k}\log n +n^{k-1}\right)}$ if $\pc = \Omega(1)$.
\end{theorem}
\begin{proof}
  The proof is identical to the proof of Theorem~\ref{thm:duplicate-elimination}, except how we handle the initial gain of diversity on the plateau.  In the case of duplicate minimisation, we have the extra property that the size of the largest species cannot increase over time. Hence, we only have to wait until the size of the largest species is at most $\speciesConstant\mu$ (instead of the entire duplicate count, as with duplicate elimination). This saves us an extra $\mu$-factor in the waiting time to reach a point where we can apply Lemma~\ref{lem:jump-to-opt}. 

Again, 
we assume that the entire population has already reached the plateau. Let $X_t$ be the count of duplicates in the population at time $t$. Let $Y_t$ denote the size of the largest species, that is, the cardinality of the largest species in the population at time $t$. Let $x$ be the offspring generated at time $t$. If $x$ belongs to one of the partitions of size $Y_t$, then the size of one of these partitions is temporarily increased before survival selection, but then duplicate minimisation ensures $Y_{t+1} = Y_t$ since one of the members of the partition of size $Y_{t} + 1$ is removed uniformly at random. In every other case, the size of the largest partition stays the same or decreases.

Thus, we only must wait until $Y_t \leq \speciesConstant \mu$ to apply Lemma~\ref{lem:jump-to-opt}. In each iteration, $X_t$ is decreased by one if a novel point is created. As stated in the proof of Theorem~\ref{thm:duplicate-elimination}, the probability of a novel mutation given a duplicate pair for parents is $\Omega(1/n)$. Conditional on $Y_t > \speciesConstant \mu$, the probability of selecting a duplicate pair and subsequently creating a novel offspring is at least 
$
\binom{Y_t}{2}/\big((n \binom{\mu}{2}\big) > \binom{\speciesConstant \mu}{2}/\big(n\binom{\mu}{2}\big) = \Omega(1/n).
$

Let $(\widehat{X}_t)_{t \geq 0}$ be the stochastic process defined by
\[
\widehat{X}_t = \begin{cases}
  X_t & \text{if $Y_t > \speciesConstant \mu$,}\\
  0 & \text{otherwise.}
\end{cases}
\]
Since $\widehat{X}_{t} - \widehat{X}_{t+1} \geq 0$ and decreases by at least one with probability $\Omega(1/n)$, the waiting time until $Y_t \leq \speciesConstant\mu$ is $\bigO(\mu n)$.
From this point and beyond, the size of the largest species is at most $\speciesConstant\mu$ and the proof is completed by applying Lemma~\ref{lem:jump-to-opt}, then Lemma~\ref{lem:time-to-plateau} for $\pc = \Omega(1)$ and Lemma~\ref{lem:plateau-arrival} for $\pc=1 - \Omega(1)$.
\end{proof}

%% file: deterministic-crowding.tex
We consider deterministic crowding as described in \cite{Friedrich:2009:ADM:1668000.1668003}. In effect, this tie-breaking rule always chooses a parent individual of the current offspring for removal in the selection phase (in particular, the offspring always survives if it is not worse than its parents). Thus, if offspring and parent have the same fitness, for mutation, the parent is always removed; for crossover, one parent uniformly chosen at random is removed.

\begin{theorem}\label{thm:deterministicCrowding}
    Consider the \muGA with $\pc = k/n$. Suppose that ties in the selection procedure are handled by deterministic crowding. Then the expected number of iterations until an optimal individual is created when running on $\jumpK$, $k = \mathrm{o}(\sqrt{n})$, is $\bigO(\mu n + n\log n + \mu\log\mu + n\eulerE^{5k}\mu ^{k + 2})$.
\end{theorem}

\begin{proof}
    According to Lemma~\ref{lem:plateau-arrival}, after $\bigO(\mu n + n\log n + \mu\log\mu)$ rounds, all individuals are on the plateau.

    The remaining part of this proof follows the ideas presented in the one from Theorem~$7$ of \cite{KotzingST:c:11:crossover}\footnote{\small In this proof, we set $\pc=k/n$ instead of bounding it only from above. This corrects an error that appears in the proof of Theorem~$7$ of~\cite{KotzingST:c:11:crossover}, in which the reciprocal of the crossover probability erroneously does not appear in the run time bound.}\!.
We want to make sure that we end up having two individuals that do not have a $0$ in common and that these get chosen for crossover that succeeds in creating the optimum.

    This process is divided into two phases. Phase $1$ considers the probability of mutating two individuals such that they end up sharing no $0$. Phase $2$ then considers the probability of a crossover to occur that chooses the two individuals of phase $1$. The phases are, too, separated into several events that are sufficient for the desired outcome.

    Phase $1$ lasts $n$ rounds without any crossover and should generate two different individuals that do not share a single~$0$. Phase $2$ goes on for $n/k$ rounds and should generate the optimum. Therefore, a harmful mutation, i.e., one that changes at least one of the two designated individuals, must not occur.
    
    We start off with phase $1$ and decompose this phase into the events E1–E5.
    
    \textbf{E1:}  The event that no crossover occurs in $n$ rounds. Its probability is $(1 - \pc)^n \geq (1 - k/n)^n \geq \eulerE^{-k}\big(1 - \mathrm{o}(1)\big)$.

    \textbf{E2:} The event, conditional on E1, that there are at least $k$ mutations during phase $1$ that flip two bits such that a $1$-bit and a $0$-bit get flipped. Note that an individual generated this way is on the plateau.

    A single such mutation happens with probability $q := k/n \cdot(n - k)/n \cdot (1 - 1/n)^{n - 2} \geq k/n \cdot 3^{-1}$ if $n$ is large enough. Trivially, $q \leq k/n$ holds.

    The probability of E2 happening is therefore at least
    \begin{flalign*}
        &\binom{n}{k}q^k (1 - q)^{n - k} \geq \frac{n^k}{k^k}\left(\frac{k}{n} \cdot \frac{1}{3}\right)^{k}\left(1 - \frac{k}{n}\right)^{n - k}
    \end{flalign*}
    \begin{flalign*}
        &\hspace*{2 em}\geq 3^{-k}\left(1 - \frac{k}{n}\right)^{\left(\frac{n}{k} - 1\right)k} \geq (3\eulerE)^{-k}\,.
    \end{flalign*}

    \textbf{E3:} The event, conditional on E1, that any mutation that creates a new individual on the plateau does so by just flipping a single $0$ (and of course a single $1$), i.e., it is unlikely that at least two $0$s get flipped, the probability of which would be at most $\binom{k}{2} \cdot 1/n^2 \leq k^2/n^2$\!.

    The probability of this never happening during phase 1 and, thus, the probability of E3 is therefore bounded by
    \[
        \left(1 - \frac{k^2}{n^2}\right)^n \geq 1 - \bigO\left(\frac{k^2}{n}\right) \overset{k = \mathrm{o}(\sqrt{n})}{\geq} 1 - \mathrm{o}(1)\ .
    \]

    \textbf{E4:} The event, conditional on E1, E2, and E3, that two designated individuals get chosen for the $k$ mutations. In the end, these individuals should be the ones chosen for crossover. The probability of choosing the correct individuals for mutation is at least $(2/\mu)^{k}\!$.

    Our tie-breaking rule is deterministic crowding, so the offspring will always survive because the parent will be removed. Hence, the probability for E4 is at least $(2/\mu)^{k}\!$.

    \textbf{E5:} The event, conditional on E1 through E4, that the two individuals from E4 actually drift apart such that, in the end, they do not share any of their $0$s. To do so, the individuals must increase their Hamming distance to one another by 2. Let $i$ denotes the number of $0$-bits that do not have to be mutated anymore. This probability is at least
    \begin{flalign*}
    &\prod_{i = 1}^{k - 1} \left(\frac{k - i}{k}\cdot\frac{n - k - i}{n - k}\right) \geq \left(\frac{n - 2k}{n - k}\right)^{k - 1}\prod_{i = 1}^{k - 1} \frac{k - i}{k}\\
    &\hspace*{1 em}\geq \left(1 - \frac{k}{n - k}\right)^{k - 1}\cdot\frac{(k - 1)!}{k^{k - 1}}\\
    &\hspace*{1 em}\geq \left(1 - \frac{k^2}{n - k}\right) \cdot \frac{\left(\frac{k}{\eulerE}\right)^k}{k^k} \geq \eulerE^{-k}(1 - \mathrm{o}(1))\ .
    \end{flalign*}
    
    We now focus on phase~$2$ and condition on the events E1 through E5.
    
    \textbf{E6:} The event that during the next $n/k$ rounds no accepting mutation occurs and that at least one crossover choosing the correct two individuals is performed.
    
    We first consider the mutations that could be harmful. Note that for such a mutation it is necessary to flip at least one $0$-bit. Thus, we calculate the probability that during each mutation none of the $k$ $0$s of any individual get flipped. The probability of this happening is at least $(1 - k/n)^{n/k} \geq \eulerE^{-1}\big(1 - \mathrm{o}(1)\big)$.

    We now look at the crossover. The probability of at least one crossover choosing the two correct individuals is at least $\big(1 - (1 - \pc)^{n/k}\big) \cdot \big(1/\binom{\mu}{2}\big) \geq (1 - \eulerE^{-1})/\mu^2 \geq \eulerE^{-1}\mu^{-2}\!$. The probability that the crossover creates the optimum is $2^{-2k}(1 - 1/n)^n \geq 2^{-2k}\eulerE^{-1}\!\!$.
    All in all, the probability of E6 is at least $\eulerE^{-3}\cdot2^{-2k}\mu^{-2}\big(1 - \mathrm{o}(1)\big)$.

    The probability of all of the events E1 through E6 happening is thus at least $\Omega(\eulerE^{-5k}\mu^{-k - 2})$, and the length of such an event is $n + n/k = \bigOh{n}$. Hence, the expected time to create the optimum, once the plateau is reached, is in $\bigO(n\eulerE^{5k}\mu^{k + 2})$.
\end{proof}

%% file: convex-hull.tex
Given two bit strings $x,y \in \{0,1\}^n\!$, uniform crossover can produce any bit string $z$ such that, for all $i \leq n$, $z_i \in \{x_i,y_i\}$; in this sense, any such $z$ is in between $x$ and $y$. Accordingly one can define the convex hull of a set $P \subseteq \{0,1\}^n$ as the set of all those bit strings which are producible with repeated application of uniform crossover. In this sense, evolutionary search with crossover means searching the convex hull of the population~\cite{MoraglioS:c:12}.

Thus, it makes sense to consider a tie-breaking rule which maximises the size of this convex hull. Since the size of the convex hull of a set $P$ of bit strings is determined by the number of positions $i$ for which there is an $x \in P$ with $x_i = 0$ and a $y \in P$ with $y_i = 1$, we can formalize this tie-breaking rule as follows. Given a population $P$ of bit strings with worst fitness, remove an individual $z \in P$ such that
\[
    \sum_{i = 1}^{n} \left[\exists x,y \in P \!\smallsetminus\! \{z\}\colon x_i = 0 \wedge y_i = 1\right]
\]
is maximised, where $[B]$ denotes the Iverson bracket (indicator function) for a proposition $B$.

\begin{theorem}\label{thm:convexHull}
    Consider the \muGA with $\mu \leq n/k$ and $\pc = 1 - \Omega(1) > 0$. Suppose that ties in the selection procedure are handled by maximising the convex hull. Then the expected number of iterations until an optimal individual is created when running on $\jumpK$, $k \leq n/2$, is $\bigO(\mu n^2\log n + 4^k \pc^{-1})$.
\end{theorem}

\begin{proof}
    First, we determine the time needed until all individuals are on the plateau. By Lemma~\ref{lem:plateau-arrival}, this is $\bigOh{\mu n + n\log n + \mu\log\mu}$. Note that $\mu\log\mu$  gets dominated by $\mu n$ because we bound $\mu$ by a polynomial in $n$.
    
    \newcommand*{\potentialFunction}{h}
    The rest of this proof is similar to the one of Theorem~3 in~\cite{Gao:2014:RAM:2576768.2598251}. We first introduce some terms that come in handy for the rest of the proof.

    We define \emph{bad} $0$s as $0$s at a bit position such that there exists another individual having a $0$ at that same position. We do not want to have such $0$s, because they contradict our goal of all individuals ultimately having their $0$s at unique positions.

    Analogously we define a \emph{good} $1$ as a $1$ at a bit position such that there exists no individual that has a $0$ at that same position. We call such a position \emph{good} as well. During mutation, good $1$s can be used to turn into $0$s that no other individual has.

    We proceed via drift analysis and define our potential $\potentialFunction$ to be the number of good positions. Let $\potentialFunction'$ denote the potential after one iteration of the algorithm, and let $A$ denote the event that $\potentialFunction' < \potentialFunction$. Note that $\potentialFunction'$ cannot increase as each good position decreases the value of the convex hull by one, thus decreasing the convex hull if we end up with more good positions than before.
    Hence, we can easily estimate
    \[
        \E[h - h' \mid h] \geq \E[h - h' \mid h, A]\Pr(A \mid h) \geq \Pr(A \mid h)\ .
    \]
    We decompose the analysis of $\Pr(A \mid h)$ into four smaller events.

    \textbf{E1:}
    The event to choose an individual with a bad $0$ for mutation. If total diversity has not been reached, there is at least one individual having a bad $0$. Thus, the probability of E1 is at least $1/\mu \cdot (1 - \pc)$.

    \textbf{E2:} The event that the mutation flips exactly one $0$ and exactly one $1$ (hence, the mutated individual is on the plateau). Note that E2 does not focus on the $0$ being a bad one and the $1$ being a good one. The probability of any two-bit-flip mutation is $k/n \cdot (n - k)/n \cdot (1 - 1/n)^{n - 2}\!.$

    \textbf{E3:} The event, conditional on E1 and E2, to choose a bad $0$ for mutation. The probability to do so is at least $1/k$.

    \textbf{E4:} The event, conditional on E1 and E2, to choose a good $1$. Due to the definition of $\potentialFunction$, the probability of E4 is $\potentialFunction/(n - k)$.

    Taking all of this together, we get that the probability of $A$, given $\potentialFunction$, is at least
    \[
        \frac{1 - \pc}{\mu} \cdot \frac{1}{n} \cdot \frac{h}{n} \cdot \left(1 - \frac{1}{n}\right)^{n - 2} = \Omega\left(\frac{h}{\mu n^2}\right)\ .
    \]
    Using the multiplicative drift theorem~\cite{Doe-Joh-Win:j:12:multiDrift}, we get a run time of $\bigO(\mu n^2\log n)$ until reaching maximal diversity.

    At last, time needed to perform crossover after reaching maximum diversity takes expected $\pc^{-1}$ rounds. Since all two different individuals have no $0$s in common, such a crossover is successful with probability $1/2^{2k}$\!. This results in an overall waiting time of $4^k \pc^{-1}$ in the end.
\end{proof}

%% file: total-hamming-distance.tex
In the section before we looked at the maximisation of the convex hull of a population $P\!$. The convex hull operator only looks, per position, for two individuals having different bits at said position.

A more thorough operator can take all the bits per individual into account to give a more detailed view on the diversity of $P\!$. Such an operator is the one maximising the total Hamming distance of all individuals in $P\!$.

For any set of bit strings $P$\!, we let 
$$
g(P) = \sum_{x \in P}\sum_{y \in P} d_{\mathrm{H}}(x, y)
$$
be the total Hamming distance of $P\!$. We consider the tie-breaking rule which, given a population $P\!$, removes an individual $z \in P$ such that $g(P \!\smallsetminus\! \{z\})$ is maximised.

\begin{theorem}\label{thm:total_hd}
    Consider the \muGA with $\mu < n/(2k)$ with $\pc = 1 - \Omega(1) > 0$. Suppose that ties in the selection procedure are handled by maximising total Hamming distance. Then the expected number of iterations until an optimal individual is created when running on $\jumpK$, $k \leq n/8$, is $\bigO(n\log n + \mu^2kn\log(\mu k) + 4^k \pc^{-1})$.
\end{theorem}

\begin{proof}
    \newcommand*{\potentialFunction}{h}
    This proof is similar to the one of Theorem~\ref{thm:convexHull} and uses the same terms of bad $0$s and good $1$s.
    
    Using Lemma~\ref{lem:plateau-arrival}, all individuals are on the plateau within $\bigOh{\mu n + n\log n + \mu\log\mu}$ steps. The $\mu\log\mu$ term is dominated by $\mu n$ because of our bound on $\mu$.

    The remaining analysis is, again, done via a drift argument. Note that the maximum of $g(P)$ is $2k\mu(\mu - 1)$; thus, we let our potential function $\potentialFunction$ be $2k\mu(\mu - 1) - g(P)$ and show that this potential reaches $0$.

    The initial potential is at most $2k\mu(\mu - 1)$ and $\potentialFunction$ cannot increase, due to the selection operator always choosing an individual to discard such that the total Hamming distance is not decreased.

    As before, let $\potentialFunction'$ denote the potential after a mutation and let $A$ denote the event that the potential decreased. We have
    $
        \E(\potentialFunction - \potentialFunction' \mid \potentialFunction) \geq \Pr(A \mid \potentialFunction).
    $
    Again, we decompose $\Pr(A \mid \potentialFunction)$.

    \textbf{E1:} The event that an individual having a bad $0$ is chosen for mutation. Each bad $0$ adds at most $2(\mu - 1)$ to $\potentialFunction$. Since each individual can have up to $k$ bad $0$s, there are at least $\potentialFunction/\big(2k(\mu - 1)\big)$ individuals having at least one bad $0$. The probability of E1 is therefore at least $\potentialFunction/\big(2k\mu(\mu - 1)\big)\cdot(1 - \pc)$.

    \textbf{E2:} The event that mutation creates an individual on the plateau by flipping a $0$- and a $1$-bit. The probability of E2 is thus at least $k/n\cdot(n - k)/n\cdot(1 - 1/n)^{n - 2}\!$.

    \textbf{E3:} The event, conditioned on E1 and E2, to choose a bad $0$ during mutation. The respective probability is at least $1/k$.

    \textbf{E4:} The event, conditioned on E1 and E2, to choose a good $1$ during mutation. We pessimistically assume that the total Hamming distance $(= 2k\mu(\mu - 1) - \potentialFunction)$ divided by $(\mu - 1)$ is the number of $1$s that are no longer good because each good $1$ adds $(\mu - 1)$ to the total Hamming distance. The mutation must choose one of the remaining good $1$s. Thus, the probability of E4 is at least
    \[
        \frac{n - k - \frac{2k\mu(\mu - 1) - \potentialFunction}{\mu - 1}}{n - k} = \frac{n + \frac{\potentialFunction}{\mu - 1} - k\left(2\mu + 1\right)}{n - k}\ .
    \]
    Note that the numerator is always nonnegative for $\mu < n/(2k)$. Because we need $\mu \geq 3$, it follows that $k \leq n/8$.

    Overall, the probability of $A$ is at least
    \begin{flalign*}
        &\frac{\potentialFunction(1 - \pc)}{2k\mu(\mu - 1)}\cdot\frac{1}{n}\cdot\frac{n + \frac{\potentialFunction}{\mu - 1} - k\left(2\mu + 1\right)}{n}\cdot\left(1 - \frac{1}{n}\right)^{n - 2}\\
        &\geq \frac{\potentialFunction(1 - \pc)\left(1 + \frac{\potentialFunction}{(\mu - 1)n} - \frac{2\mu + 1}{n}k\right)}{2\mu^2kn}\eulerE^{-1} \geq \frac{\potentialFunction(1 - \pc)}{2\mu^2kn}\eulerE^{-1}\,,
    \end{flalign*}
    which is positive as long as maximal diversity ($\potentialFunction = 0$) has not been reached. Note that this is our desired drift and that $1 - \pc = \Omega(1) > 0$.

    We can now bound the expected time until maximal diversity has been reached by using the multiplicative drift theorem~\cite{Doe-Joh-Win:j:12:multiDrift}. This yields an expected number of rounds in $\bigO\big(\mu^2kn\log(\mu k)\big)$.
    
    Now all that is left is that a crossover is performed that generates the optimum. Again, the expected number of steps for this event is $4^k\pc^{-1}\!\!.$ This completes the proof.
\end{proof}

%% file: fitness-sharing.tex
We consider a tie-breaking rule that makes use of fitness sharing as described in \cite{Friedrich:2009:ADM:1668000.1668003} rather than how it is used in \cite{OlivetoSudholtZarges2014}. 
The actual fitness $f(x)$ of an individual $x$ gets skewed by how similar it is with respect to a certain measure $d$ to other individuals. The new shared fitness $\bar{f}(x)$ is the basis of the tie-breaking mechanism, where we just delete one individual with worst shared fitness.

The general scheme of the fitness sharing mechanism is parameterised with a metric $d$ and two numbers $\alpha$ and $\sigma\!$. Given a population $P$ of individuals with worst fitness, we want to remove an individual $z \in P$ such that
\[
    \sum_{x \in P \smallsetminus \{z\}} \frac{f(x)}{\sum_{y \in P \smallsetminus \{z\}} \max\left\{0, 1 - \left(\frac{d(x, y)}{\sigma}\right)^{\alpha}\right\}}
\]
is maximised. That means that there is a penalty for the similarity of $x$ to all other individuals $y$ up to a distance of~$\sigma$; $\alpha$ determines the shape of the penalty.

In this paper, we consider the Hamming distance, i.e., $d = d_{\mathrm{H}}$, and we set $\alpha = 1$, that is, we have a linear penalty. Note that we only use this fitness sharing rule for breaking ties (this ensures that the initial climb to the plateau is undisturbed). We call this tie-breaking rule \emph{Hamming fitness sharing}.

\begin{lemma}
    \label{lem:fitness_sharing}
    Given a population $P$ of individuals all having the same fitness; suppose that any two individuals of $P$ differ by at most $\sigma\!$. Then the Hamming fitness sharing tie-breaking rule maximises the total Hamming distance of $P\!$.
\end{lemma}

\begin{proof}
    Since all individuals $x \in P$ have the same fitness, we can ignore the impact of the fitness on the tie-breaking rule. Furthermore, since individuals can differ by at most $\sigma\!$, the maximum in the fitness sharing expression is not necessary. So we want to remove an individual $z \in P$ such that
    $
        \sum_{x \in P \smallsetminus \{z\}} \big(\sum_{y \in P \smallsetminus \{z\}} (1 - \frac{d_{\mathrm{H}}(x, y)}{\sigma})\big)^{-1}
    $
    is maximised.
    That is the same as \emph{minimising} the term $
        \sum_{x \in P \smallsetminus \{z\}} \sum_{y \in P \smallsetminus \{z\}} -d_{\mathrm{H}}(x, y)$
    since we just change the monotony and remove the offset of $\sum_{y \in P \smallsetminus \{z\}} 1$ and the common factor $1/\sigma$; this is just the same as maximising the total Hamming distance of $P\!$.
\end{proof}

Thanks to this lemma, we now know that Hamming fitness sharing and maximising the Hamming distance are equivalent. Thus, we can immediately carry over all run time results from the maximisation of the Hamming distance to fitness sharing, as stated in the following theorem.

\begin{theorem}
\label{thm:fitness_sharing}
    Consider the \muGA with $\mu < n/(2k)$ with $\pc = 1 - \Omega(1) > 0$. Suppose that ties in the selection procedure are handled by Hamming fitness sharing with $\sigma \geq 2k$. Then the expected number of iterations until an optimal individual is created when running on $\jumpK$, $k \leq n/8$, is $\bigO(n\log n + \mu^2kn\log(\mu k) + 4^k \pc^{-1})$.
\end{theorem}

\begin{proof}
    The time needed to reach the plateau is by Lemma~\ref{lem:plateau-arrival} in $\bigOh{\mu n + n\log n + \mu\log\mu}$. As we bound $\mu$ by a polynomial in $n$, the $\mu\log\mu$ term is dominated by $\mu n$.
    
    After reaching the plateau, Lemma~\ref{lem:fitness_sharing} yields that on the plateau Hamming fitness sharing is just maximisation of total Hamming distance. Thus, the statement follows from the proof of Theorem~\ref{thm:total_hd}.
\end{proof}

%% file: island-model.tex
An easy way to ensure diversity in a population is to keep different parts strictly separate. Previous results~\cite{Neumann2011} have shown that island models can provide enough diversity that can be subsequently leveraged by crossover. We consider a single-receiver island model, described in~\cite{Watson2007}, as follows. We have $\mu + 1$ islands, $\mu$ of them running a \oneoneea independently from one another. Furthermore, there is one island which is the receiver island. Each iteration, it chooses two of the $\mu$ islands uniformly at random, copies their respective best-so-far individuals, and then performs uniform crossover on those, hence $\pc=1$. The resulting offspring replaces the resident individual if it has higher fitness. We say that the island model succeeds if the receiver island produces the optimum. We do not employ any particular tie-breaking rule on the islands (and, in the case of a tie, choose a survivor uniformly at random).

\begin{theorem}\label{thm:islandModel}
    Consider the island model with $\pc=1$, $\mu = \bigOh{n^c}$ islands for any $c$, and $\mu \geq 2$. For optimizing $\jumpK$, $k = \smallOh{\sqrt{n}/\mu}$, the expected run time until the optimum is produced is $\bigOh{n\log n + \mu^2kn + \mu^2 4^k}$.
\end{theorem}

\begin{proof}
	In this proof, we follow the same language as in the proof of Theorem~\ref{thm:convexHull}.

    In expectation, after $\bigOh{n \log n^d}$ steps, for a constant $d > 0$, a single \oneoneea is on the plateau with probability $1 - 1/n^d$~\cite{doerr2013adaptive}. Hence, the probability of $\mu$ independently running \oneoneeas being \emph{all} on the plateau after $\bigOh{n \log n^d}$ steps is $\big(1 - 1/n^d\big)^\mu$\!, which goes toward $1$ as $n$ goes to infinity, because of our constraint on $\mu$ and because we can choose $d > c$.

    Once all \oneoneeas are on the plateau, we proceed via drift analysis: we define a potential for the island model and show that there is a bias toward $0$. Fix two islands under consideration. The potential $X_t$ at point $t$ (starting from the plateau) is defined as follows: it is $0$ if the receiver island produced the optimum. If not, we have a look at the individuals of the two fixed islands. Assume that these individuals have $i \leq k$ of their $0$s in common. $X_t$ is then $i\mu^2 n\eulerE^2 + \mu^2 4^k\eulerE$.

    The drift is the expected change in potential, i.e., $\E(X_t - X_{t + 1} \mid X_t)$. We make a case distinction to whether the potential decreases or increases. %
		In order to compute the drift, we assume $X_t > 0$ to be given. Let $i$ be such that $X_t = i\mu^2 n\eulerE^2 + \mu^2 4^k\eulerE$.

		First, we consider $i > 0$. Whenever the potential decreases, it decreases by at least %
		$\mu^2 n\eulerE^2\!$.

    We now lower-bound the probability of the potential decreasing, i.e., the two individuals increase their Hamming distance: we assume that only one individual flips a bad $0$ and a good $1$ and that our two individuals are chosen. No other bits are mutated. Therefore, the probability is at least
    \[
        \frac{2}{\mu^2}\left(1 - \frac{1}{n}\right)^{n - 2}\!\!\!\cdot\frac{n - 2(k - i) - i}{n}\cdot\frac{i}{n}\left(1 - \frac{1}{n}\right)^{n} \geq \frac{1}{\mu^2 n\eulerE^2}\ .
    \]
    Thus, the overall positive drift is at least $1$.

    If we now consider $i = 0$, we easily get a positive drift of at least $1$ as well, since then all that is left to create the optimum is to have the uniform crossover always choose the correct bits of the $2k$ positions where the two individuals differ and not mutating any bit after crossover, resulting in a probability of at least $1/(\mu^24^k\eulerE)$.

    We now upper-bound the negative drift. We only consider the case that $i$ increases by $1$. An increase by more than $1$ is possible but far more unlikely since more bit flips have to occur, leading to an additional factor of $1/n^2$ in the probability for each additional $1$. We make up for the so lost terms by multiplying our negative drift with a constant $c$.

    The change in potential conditioned on a decrease as we defined it is $-\mu^2n\eulerE^2\!$, and the probability of such a decrease is at most $ck^2/n^2$: each of the two individuals can only decrease the distance to the other one by flipping one of its own $0$s and flipping a $1$ where the other individual has a $0$.
    The absolute value of the negative drift is thus at most $c(\mu k\eulerE)^2/n$. Using our assumption regarding $k$, this results in a negative drift in $\smallOh{1}$.
    The additive drift theorem~\cite{He:2004:SDA:982675.982699} hence yields an expected run time of $\bigOh{\mu^2 kn + \mu^2 4^k}$ when starting from the plateau.
    Taking the time needed for all islands to be on the plateau into account, we get the desired run time of $\bigOh{n\log n + \mu^2kn + \mu^2 4^k}$.
\end{proof}

%% file: experiment-crossover.tex
Figure \ref{fig:impactcrossover} (a) and (b) depict the performance of the GA
($\pc=1.0$) compared to the algorithm using only mutation ($\pc=0.0$) under the same
setting ($\varpm=1/n$). The range of $n$ in this experiment is set to $n=[50
\dots 300]$ with a step size of $10$, and $k$ is in $\{2,3\}$. Even with these
small values of $k$ and $n$, a strong reduction of the average run time can be
observed, up to a multiplicative factor of $10^4$\!.

\begin{figure*}[h]
\centering
\includegraphics[width=0.8\textwidth]{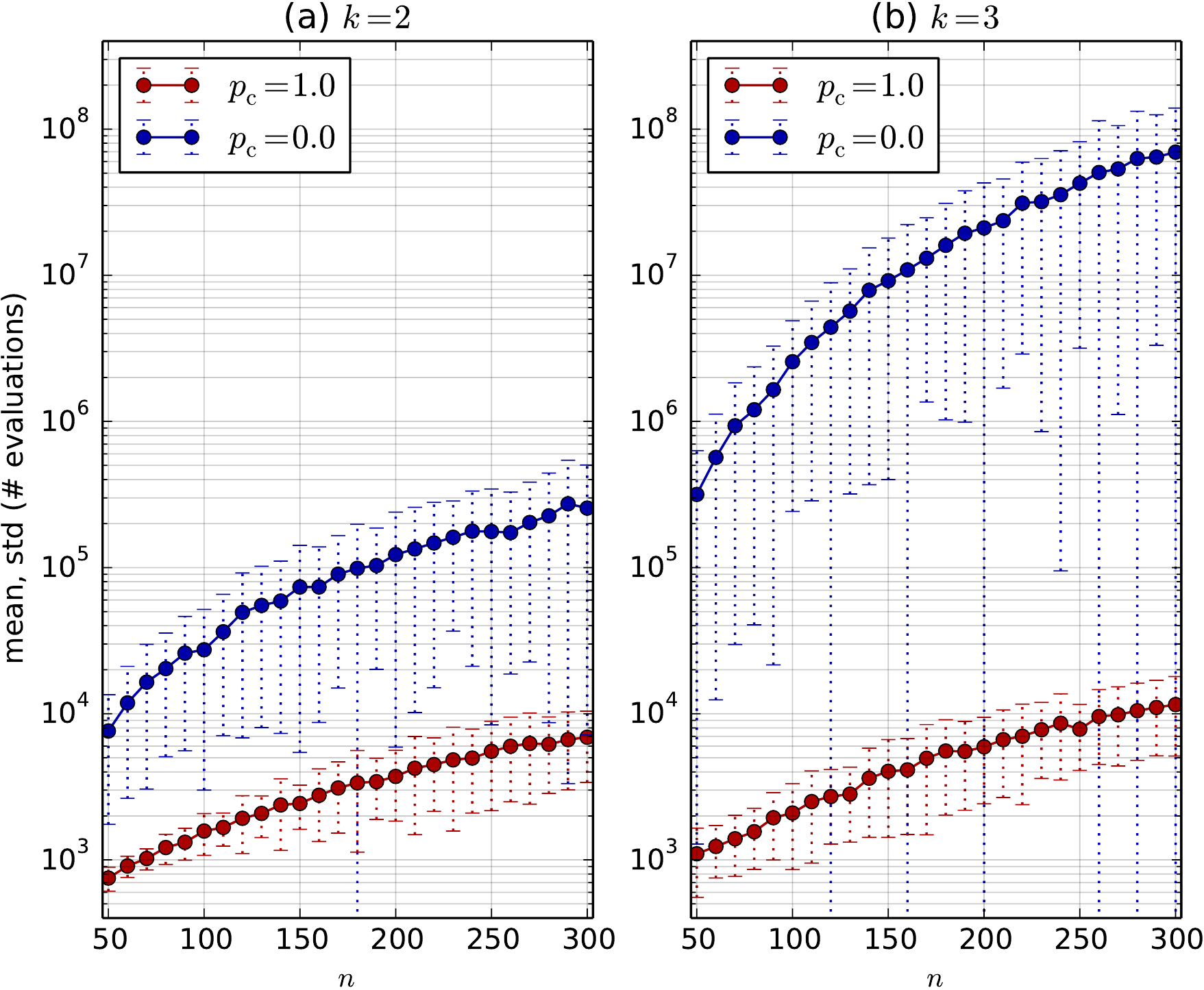}
\caption{Impact of enabling crossover.}\label{fig:impactcrossover}
\end{figure*}

The impact of the jump length $k$ on the run time is illustrated in Figure
\ref{fig:impactmutationrate} (a). The experiment was set with $n$ in $[100 \dots
5000]$ (with a step size of $100$) and $k$ is in $\{3,4,5\}$. We notice  that
the increase of $k$ does not imply a large change in the average run time. The
average run time seems to still scale linearly with $n$ in this setting even for $k=4$.
By fixing $k=3$, we also experimented with different mutation rates, i.e., $\varpm$ in
$\{0.9/n, 1/n, 1.1/n, 2/n\}$. The results are displayed in Figure~\ref{fig:impactmutationrate}
(b). We notice that the mutation rates above $1/n$ reduce the average run time
while a slightly lower mutation rate increases it considerably. With mutation
rate $2/n$, the average run time and the stability of the runs are distinctively
improved.

\begin{figure*}[h]
\centering
\includegraphics[width=0.8\textwidth]{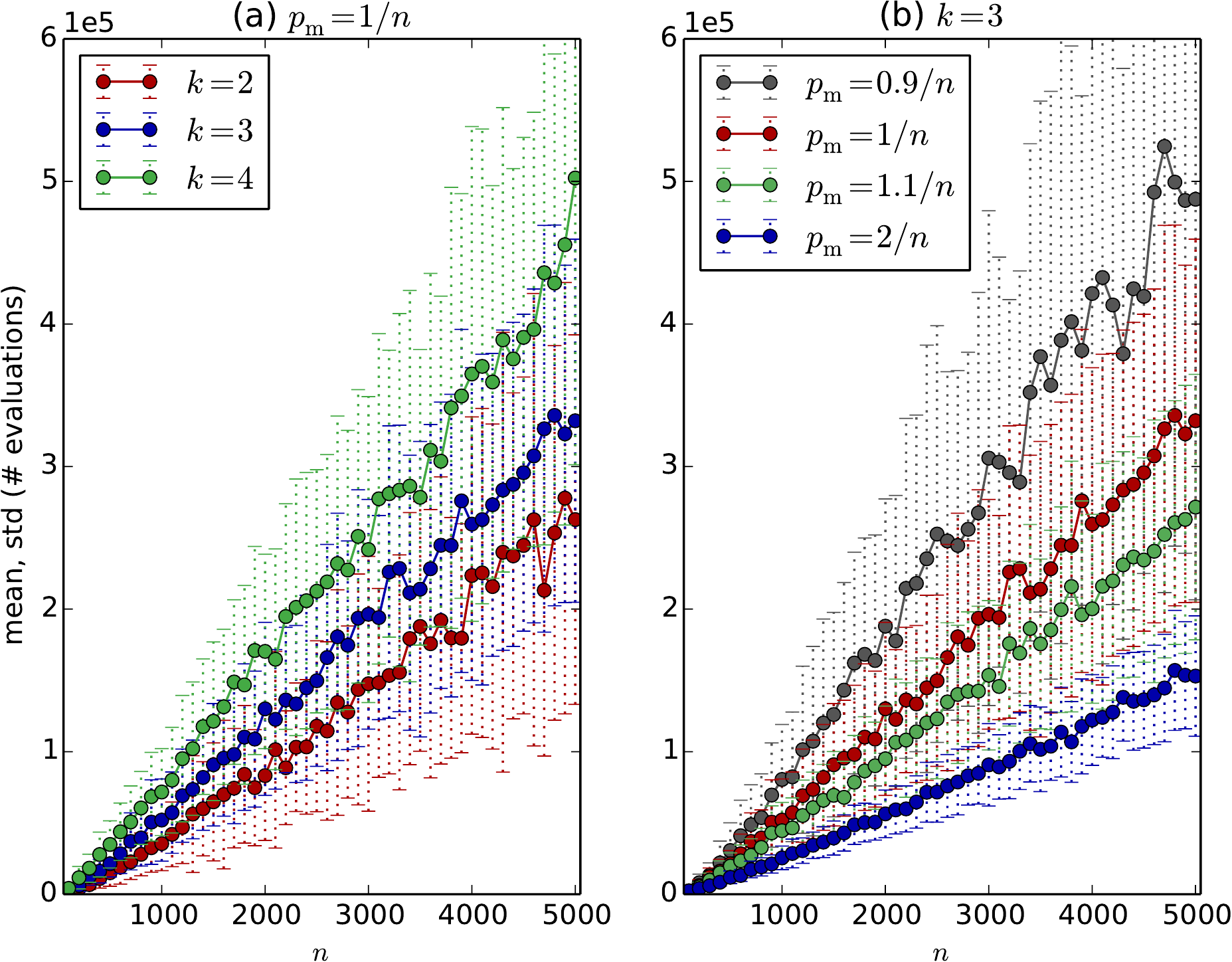}
\caption{Run time for different jump lengths $k$ and different mutation rates $\varpm$ with crossover.}\label{fig:impactmutationrate}
\end{figure*}

On the other hand, an excessive increase of the mutation rate may deteriorate the average
run time because of the likelihood of multiple bit flips which imply harmful mutations.
This can be observed in the experiment depicted in Figure~\ref{fig:impactmutationrate2}
(in log-scale) for $n=500$. In this experiment, $k$ is in $\{2,3,4\}$, and the range of
$\chi = \varpm \cdot n$ is set to $[0.6\dots 8]$ (with a step size of $0.1$). We note
that the more $k$ is increased, the stronger the negative effect of high mutation rates
can be noticed. Moreover, too low mutation rates are also bad for the run time. This
can be related to our theoretical analysis, in which a low mutation rate could have
made the random walk associated with the size of the largest species biased toward the
wrong direction. This may lead to the reduction of the population diversity and the
loss of benefit from crossover.

\begin{figure}[h]
\centering
\ifarxiv
\includegraphics[width=0.8\textwidth]{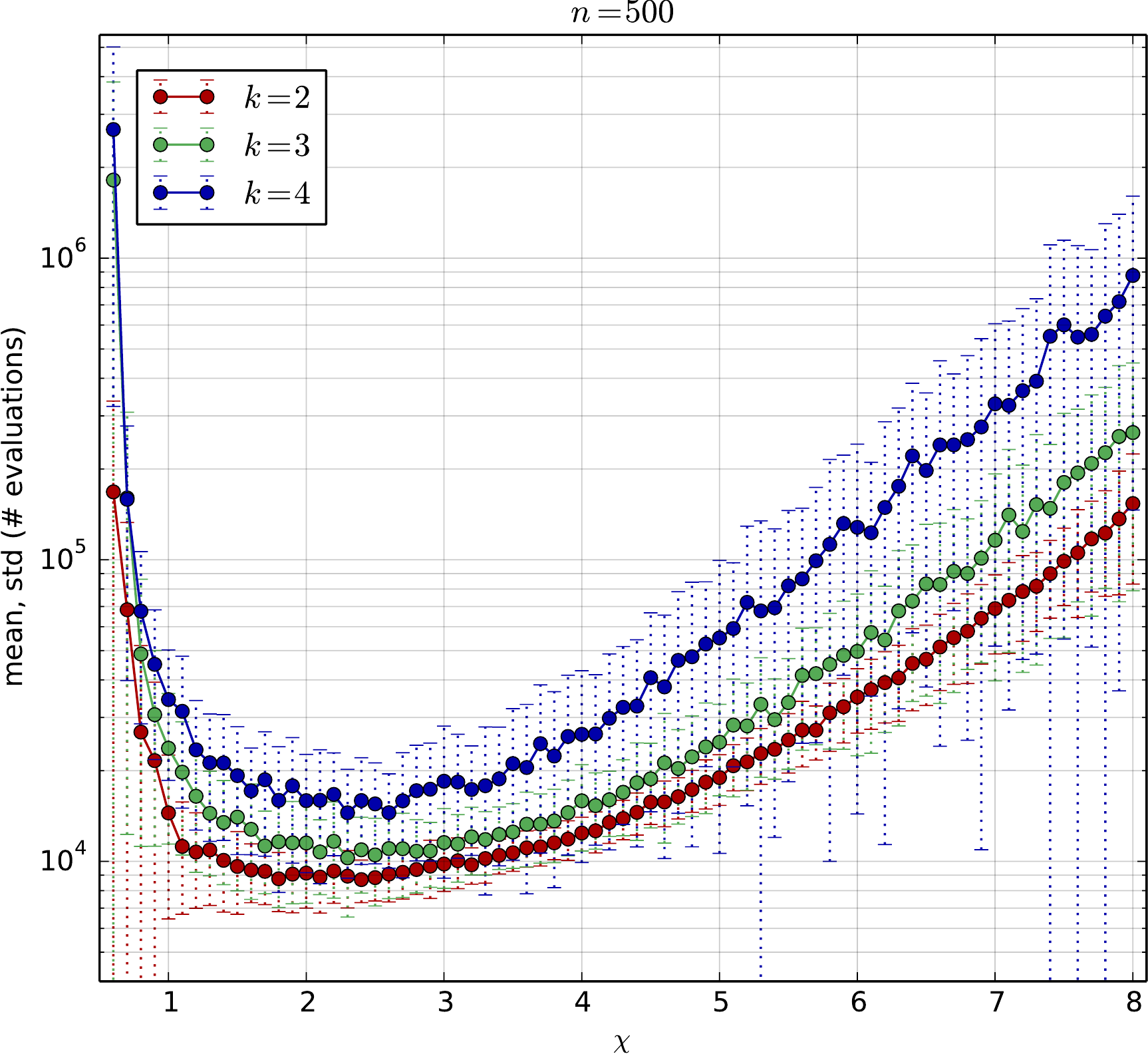}
\else
\includegraphics[width=0.5\textwidth]{Fig3.pdf}
\fi
\caption{Impact of different mutation rates $\varpm=\chi/n$ with crossover for problem size $500$.}\label{fig:impactmutationrate2}
\end{figure}

%% file: experiment-diversity.tex
Full crossover is enabled ($\pc=1.0$), the problem size $n$ is varied in 
$[100, 1000]$ (with a step size of $25$), and the standard mutation rate 
($\varpm=1/n$) is used. 
In each tested setting, 
the run is replicated 
$100$ times with different random seeds and the number of function evaluations, 
denoted as `\# evaluations', is reported as the run time. 
The result for $k=4$ is shown in 
Figure~\ref{fig:diversitymechanism}.

On average, the 
highest contribution to the reduction of the run time in order is fitness sharing,
then convex hull 
maximisation, deterministic crowding, and, finally, duplicate elimination and 
minimisation have quite similar average run times. We also notice that the the 
island model with $\mu = 2$ requires approximately the same average number of 
evaluations as deterministic crowding. Overall, compared to the standard 
\muGA, all the diversity mechanisms contribute to the reduction of the average 
run time, as well as to the stability of the result. 
In addition, we also compare the diversity mechanisms to the high mutation rate 
setting $\varpm=2.6/n$ (the best choice for $n=500$ and $k=4$, suggested by 
Figure~\ref{fig:impactmutationrate2}) of the \muGA. The high mutation rate 
setting is able to perform better than four diversity mechanisms and it is 
only less efficient than convex hull maximisation and fitness sharing.

\begin{figure*}[ht]
\centering
\includegraphics[width=0.8\textwidth]{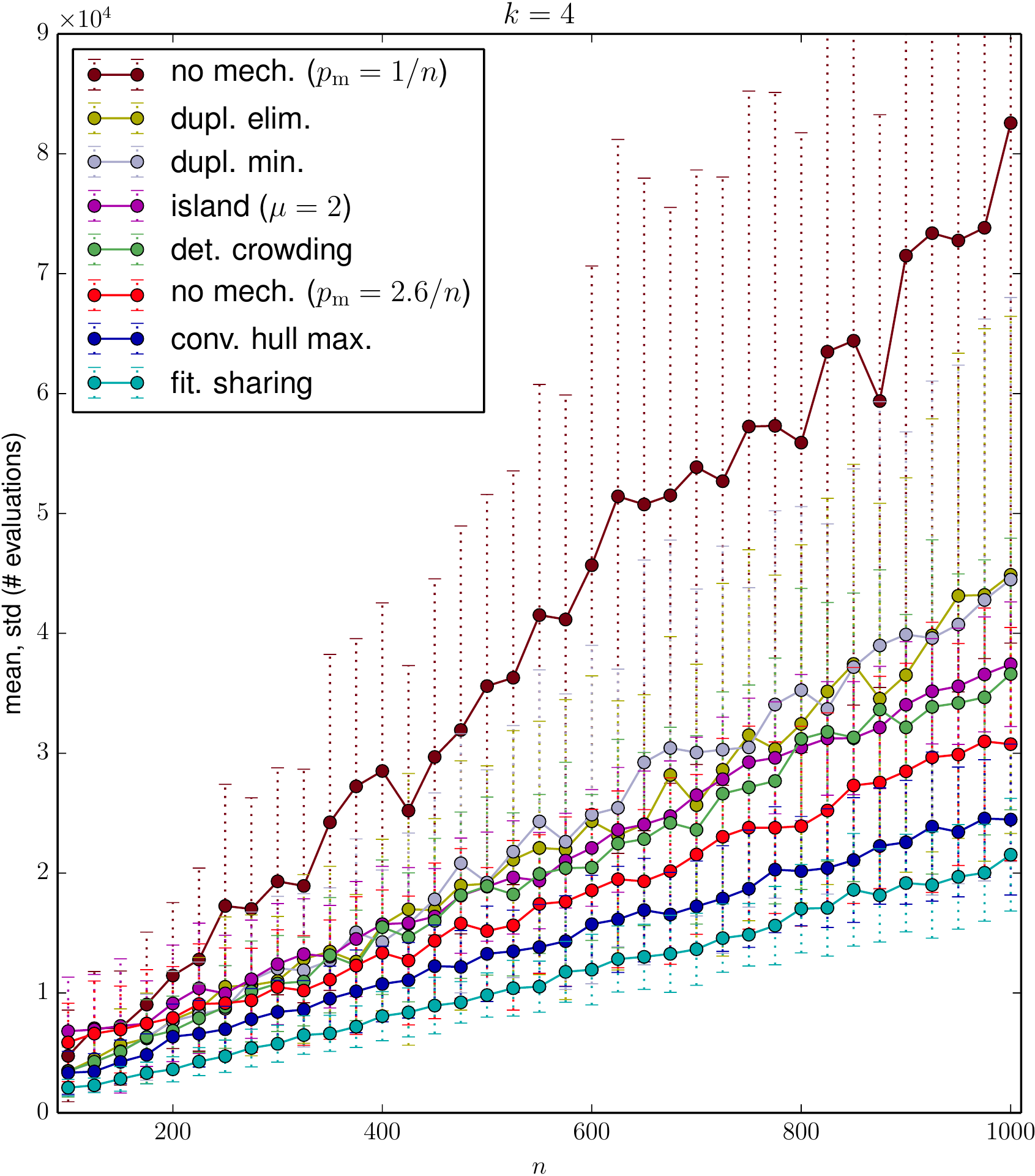}
\caption{Performance of the diversity mechanisms for jump length $4$, the mutation rate $\varpm$ is set to $1/n$ unless specified.}\label{fig:diversitymechanism}
\end{figure*}